\documentclass[preprint,1p,times,AutoFakeBold]{elsarticle}
\usepackage{caption}
\usepackage{lineno,hyperref}
\usepackage{graphicx}
\usepackage{ulem}
\usepackage{color}
\usepackage{epstopdf}
\modulolinenumbers[5]
\usepackage{multirow}
\usepackage{float}
\journal{Journal of \LaTeX\ Templates}
\usepackage{amsmath}
\usepackage{subfigure}
\usepackage[linesnumbered,ruled,vlined]{algorithm2e}

\usepackage{algpseudocode}
\usepackage{bm}
\usepackage{ntheorem}

\newtheorem{definition}{Definition} 
\newtheorem{lemma}{Lemma}
\newtheorem{example}{Example}

\newtheorem{theorem}{Theorem}
\newtheorem*{proof}{Proof}

\newcommand{\tabincell}[2]{\begin{tabular}{@{}#1@{}}#2\end{tabular}}

\begin{document}

\begin{frontmatter}

\title{PSpan: Mining Frequent Subnets of Petri Nets\tnoteref{mytitlenote}}



\author[2,1]{Ruqian Lu \corref{cor1}}
\ead{rqlu@math.ac.cn}
\cortext[cor1]{Corresponding author}
\address[2]{Academy of Mathematics and Systems Science Key Lab of MADIS Chineses Academy Science}

\author[1,3]{Shuhan Zhang}
\ead{zhangshuhan@ict.ac.cn}
\address[1]{Key Laboratory of Intelligent Information Processing of Chinese Academy of Sciences,\\ Institute of Computing Technology, Chinese Academy of Sciences}
\address[3]{University of Chinese Academy of Sciences}

\begin{abstract}
This paper proposes for the first time an algorithm PSpan for mining frequent complete subnets from a set of Petri nets. We introduced the concept of complete subnets and the net graph representation. PSpan transforms Petri nets in net graphs and performs sub-net graph mining on them, then transforms the results back to frequent subnets. PSpan follows the pattern growth approach and has similar complexity like gSpan in graph mining. Experiments have been done to confirm PSpan’s reliability and complexity. Besides C/E nets, it applies also to a set of other Petri net subclasses.

\end{abstract}


\end{frontmatter}

\section{Introduction} \label{sec:intro}

Frequent subgraph mining (FSM) is a well-studied subject in data mining \cite{kuramochi_frequent_2001}. FSM has been successfully applied to analyze protein-protein interaction networks \cite{elseidy_grami:_2014}, chemical compounds \cite{inokuchi_apriori-based_2000}, social networks \cite{leung_exploring_2010}, and mine workflow nets from large sets of log data of processes \cite{greco_mining_2005,greco_mining_2006,chapela-campa_towards_2017}. Petri nets have been extensively used to model and analyze concurrent and distributed systems \cite{murata_petri_1989,lin_properties_2001,desel_what_2001}. Despite the success of FSM in graphs, little effort has been spent on applying FSM in Petri nets that have complicated topological structures and semantics. 

This paper is devoted to the mining of frequent subnets from sets of Petri nets efficiently. A Petri net is different from a graph in three main aspects: (1) Topological structures. A graph is a net consisting of a set of nodes and another set of edges connecting the nodes, while a typical Petri net consists of two sets of different kinds of nodes and a set of directed arcs. Roughly speaking, we call the first kind of nodes the state nodes and the second kind the action nodes. (2) Connections. There are only arcs between different kinds of nodes. The set of all arcs form the flow relation. This problem makes the structure of Petri nets even more complicated than graphs and is the main source of high complexity in net information processing \cite{yuan_application_2013}. (3) Semantics. Unlike the graphs with no semantics (or, which can be assigned any appropriate semantics), Petri nets have built-in semantics. This characteristic can be seen from their names, e.g., condition/event nets, place/transition nets. Therefore a Petri net cannot be decomposed or reconfigured at will. We call this kind of semantics the integrity principle. Prof. Petri himself has pointed out in a lecture \cite{petri_concepts_1973,peterson_petri_1977} that violating this principle would lead to meaningless results. As an example, he mentioned the process of bike assembling. It can be modeled by a Petri net, see figure ~\ref{bike11}, where front-wheel + bike frame + rear-wheel is assembled to a complete bike. However, figure \ref{bike22} (assembling a bike with one wheel only) is not a meaningful subnet of it (though, we must exclude the case of a circus). 

\begin{figure}[htbp]
	\centering
	\subfigure[A bike assembling process (net)]{

		\includegraphics[width=0.4\textwidth]{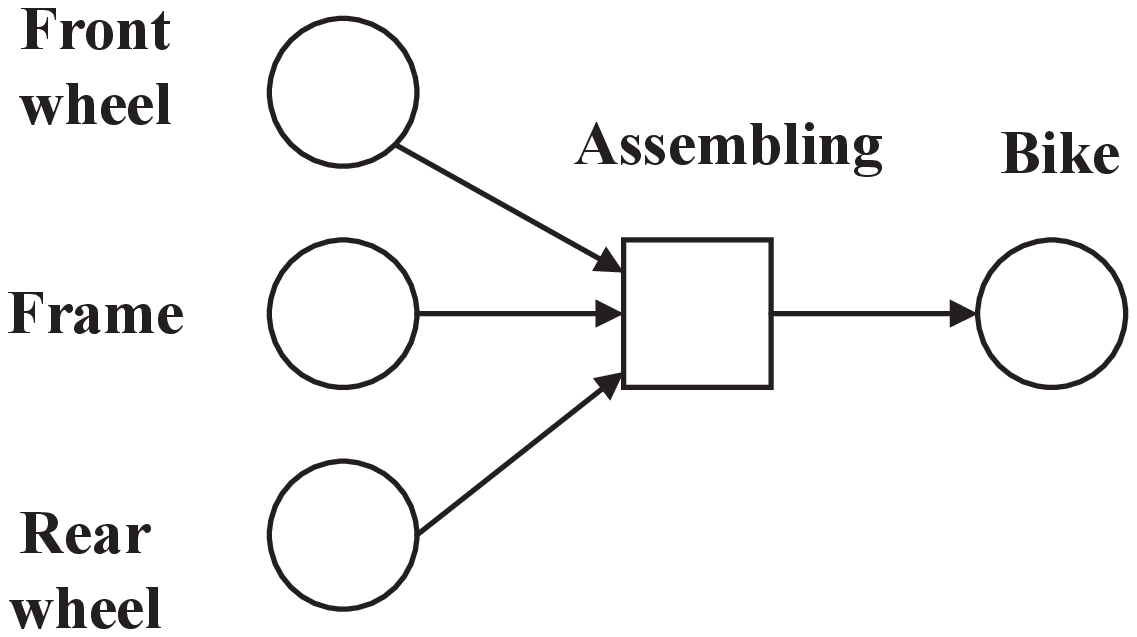}
	\label{bike11}}
	\hspace{0.2in}
	\subfigure[What's it? An incomplete subnet]{
		\includegraphics[width=0.4\textwidth]{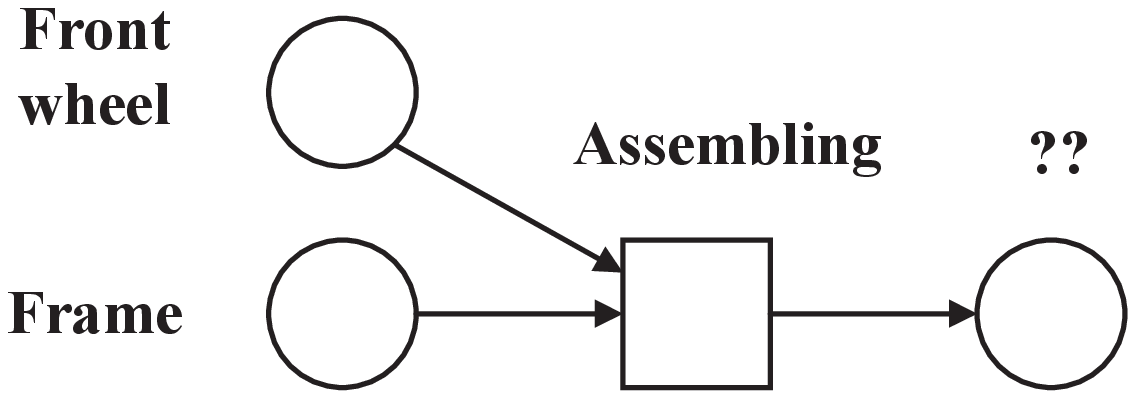}\label{bike22}}

	\caption{Semantic inadequacy of incomplete subnets}
		\label{new bike}
\end{figure}

The conclusion we draw from the integrity principle is the criterion of module oriented frequent subnet mining. For the Petri nets, a basic module is an action node $e$ together with all state nodes connected to $e$, each by one arc only. We call such a module or the connections of several such modules a complete subnet. Only complete subnets are eligible to be mined as frequent subnets. For details, see definition \ref{net}.

By designing an algorithm for subnet mining, our first concern is how to lower down the complexity of Petri net structures \cite{lu_petri_1994}. To do that, we transform a Petri net into a form called a net graph, N.G. for short. N.G. is more similar to a planar graph \cite{garey1974some}, with the advantage of a high compression ratio of node and edge numbers. A net graph has only one kind of nodes, which correspond to the action nodes of Petri nets. These nodes are connected by edges. The information of Petri nets’ state nodes and their connection arcs with action nodes is absorbed into the tagging of net graphs’ nodes and edges. We have designed and implemented a frequent sub-net graph mining algorithm called PSpan, a modification of the gSpan algorithm \cite{xifeng_yan_gspan:_2002}. The mined frequent sub-net graphs will be transformed back to frequent complete subnets of Petri nets. We use a very basic model of Petri net—C/E net (Condition/Event net) in the whole process. Section \ref{sec:extension} shows that our PSpan algorithm is applicable to quite a few other Petri net models for frequent subnets mining without loss of generality. To verify our method, we use a probabilistic algorithm to generate a large set of C/E nets randomly. The results have been satisfactory.

In the following, we first list some basic definitions.

\begin{definition}[Net]\label{net}
	A \textbf{C/E net} ( simply a \textbf{net}) is a tuple $(C,E;F)$ where:\\
	1. $C$ is a finite set, whose elements are called conditions;\\
	2. $E$ is a finite set, whose elements are called events satisfying $C \cap E = \emptyset$ and $C \cup E \not= \emptyset$;\\
	3. $F \subseteq (C \times E) \cup (E \times C)$ is the set of directed arcs, called flow relations;\\
	4. $\bullet x=\{y|(y,x) \in F\}$ is called the pre-set of $x$, $x \bullet=\{y|(x,y) \in F\}$ is called the post-set of $x$, here $x \in C \cup E$.
	
\end{definition}

$\hfill \square$

\begin{definition}[Pure net]
	Let $K=(C,E;F)$ be a net, we call $K$ a pure net, if $K$ does not have a self-loop, i.e., there is no $a,b$, such that $\{(a,b),(b,a)\}\subseteq F$.

\end{definition}

$\hfill \square$
\begin{definition}[Connected net]
	Given a net $N=(C,E;F), \forall a,b \in C \cup E$, we call:\\
	1. $(a,b)$ is connected iff $(a,b) \in (F \cup F^{-1})^+$ holds, where $(b,a) \in F^{-1}$ iff $(a,b) \in F$, and $R^+$ denotes the transitive closure of relation set $R$;\\
	2. $(a,b)$ is strongly connected iff $(a,b) \in aF^+b$, i.e., there exists a directed path leading from $a$ to $b$; \\
	3.  $N$ is connected (strongly connected) iff $\forall a,b \in C \cup E, (a,b)$ is connected (strongly connected).
\end{definition}
$\hfill \square$
\begin{definition}[Complete subnet]
	Given two nets $N_1=(C_1,E_1;F_1)$ ; $N_2=(C_2,E_2;F_2)$, $N_2$ is a complete subnet of  $N_1$ iff the following limitations hold.\\
	1. $C_2 \subseteq C_1, E_2 \subseteq E_1, F_2 = F_1 \cap ((C_2\times E_2)\cup (E_2 \times C_2))$;\\
	2. $\forall e \in E_2,\{(e,b),(c,e)\} \subseteq F_1 \to \{(e,b),(c,e)\}\subseteq F_2 \,and\, \{b,c\}\subseteq C_1$.\\
Note that if $|E_2|=n$, and $N_2$ is connected, we call $N_2$ an $n$-complete subnet of $N_1$.	
\end{definition}

$\hfill \square$
\par

\begin{definition}[Net isomorphism] \label{iso}
	Given two nets $N_1=(C_1,E_1;F_1)$ and $N_2=(C_2,E_2;F_2)$.  An isomorphism between C/E nets $N_1$ and $N_2$ means there is a mapping $m:N_1 \to N_2$, which is bijective in $C_1 \to C_2, E_1 \to E_2$ and $F_1\to F_2$ respectively, where arc $(x,y) \in F_1$ iff arc $(m(x),m(y))\in F_2$.\\
	In this case, we say $N_1$ is $m$-isomorphic to $N_2$, or $N_2$ is $m^{-1}$-isomorphic to $N_1$.
\end{definition}

$\hfill \square$

\par
The remainder of the paper is arranged as follows: In Section~\ref{sec:related}, we review related works, including process mining techniques, FSM algorithms, and their applications on workflow nets mining. Section ~\ref{sec:preliminaries} introduces the net graph representation of pure C/E nets and presents our PSpan algorithm. In particular, we describe details on transforming net representation into N.G. form, and on a minimal-depth-first search strategy of constructing minimal DFS codes of N.G.. The PSpan algorithm will be presented in details. A complexity analysis of PSpan is illustrated in section \ref{complexity}. In section ~\ref{sec:experiment}, we perform experiments on subnet mining with the PSpan algorithm on the generated pure C/E nets reservoir and compare PSpan’s complexity with an ideal experimental algorithm DSpan. In section~\ref{sec:extension}, we discuss the extension of our PSpan algorithms to other subclasses of Petri nets. Finally, Section ~\ref{sec:Conculsion} closes this paper with a summary.

\section{Related Works} \label{sec:related}

In this section, we review shortly three areas of related works. (1) Existing process mining algorithms that focus on mining workflow nets from process data; (2) Frequent subgraph mining (FSM) algorithms; (3) Applications of using FSM on frequent workflow nets mining.
%

\subsection{About Process Mining}
Process mining is to mine process patterns from a large set of log data. The models for the mined processes may be any concurrency structures, mostly different types of Petri nets. From the 1960s, Petri nets have been extensively used to model and analyze concurrent and distributed systems \cite{petri_kommunikation_1962,_three-layer_2007,silva_half_2013}.
There are two main aspects of Petri net mining from log data: ordinary Petri net mining and workflow net mining. Tapia-Flores et al. \cite{tapiaflores_discovering_2018} proposed Safe Interpreted Petri net mining by identifying the casual occurrence relations in logs for computing the T\_invariants. However, the input/output conditions in this kind of Petri nets are subject to some limitations (i.e., 1-bound Petri net). The most commonly used technique is the automatic process mining approach, which focuses on mining workflow nets (a special kind of Petri net) from process data \cite{aalst_workflow_2004,wen_mining_2007,wen_novel_2009,vander_aalst_decomposing_2013}. The goal of mining workflow nets from process data is to utilize the causal information collected at the system running to derive a single-entry and single-exit workflow net. This technique is motivated by data mining and business process management (BPM) discipline. There are two main approaches:

(1)Rule-based matching approach. The most popular rule-based approach is $\alpha$-algorithm \cite{aalst_workflow_2004}. It distinguishes four log-based ordering relations, including direct successions, causality dependencies, parallel dependencies and irrelevant. According to these four relations, $\alpha$-algorithm extracts the directly-follows graphs from event logs first and then transforms them into workflow nets that preserve some specific properties (i.e., soundness). Some commercial and open-source tools, such as the ProM \cite{van_der_aalst_prom:_2009}, Disco \footnote{http://fluxicon.com/disco/}, and Celonis \footnote{https://www.celonis.com/intelligent-business-cloud/process-discovery/}, have been commonly used. However, some spaghetti-like workflow nets are generated from a large number of process data.

(2) Sequential pattern mining (SPM)-based approach. SPM techniques have been successfully integrated into workflow nets mining from process data
 \cite{chapela-campa_discovering_2017,dalmas_heuristic_2018,diamantini_pattern_2013,leemans_discovery_2015,tax_mining_2016,tax_localprocessmodeldiscovery:_2018}. Leeamans et al. \cite{leemans_discovery_2015} proposed the Episode Discovery algorithm to discover frequent episodes that are partially ordered sequential data. A formal method called Local Process Model has been proposed to detect the local patterns such as concurrency and choice ones in process data \cite{tax_mining_2016}. As one can easily see, the sources of both rule-based matching approaches and SPM-based approaches are sequential log data. One of the significant differences between all other approaches and ours is that we mine frequent subnets from a large set of Petri nets. Figure ~\ref{1} illustrates the two main methods of mining workflow nets from process data. Note that the label “B 50/85” means there exist 85 event logs containing B, while only 50 event logs ($<A,B,C>^{20},<A,B,D>^{30}$) can replay the workflow net in Figure ~\ref{fig:subfig:1c}.
 
  All works mentioned above relate to process mining or Petri net mining, in particular workflow mining, from log data. One of the significant differences between all these works and ours is that we mine frequent subnets from a large set of Petri nets.
 
\begin{figure}[H]
	\subfigure[The event logs]{
		\begin{minipage}[c]{.3\linewidth}
			\centering
			\label{fig:subfig:1a}
			\begin{tabular}{c}
				\hline
				Event logs\\
				\hline
				$<A,B,C>^{20}$ \\ 
				$<A,B,D>^{30}$ \\ 
				$<B,C>^{15}$ \\
				$<B,D>^{10}$ \\
				$<B>^{10}$ \\
				$<C>^{30}$ \\
				$<D>^{40}$ \\
				\hline
			\end{tabular}
			
	\end{minipage}}
	\begin{minipage}[c]{.8\linewidth}
		\centering
		\subfigure[Workflow net generated by rule-based matching approaches]{
			\label{fig:subfig:1b} 
			\includegraphics[width=0.7\textwidth]{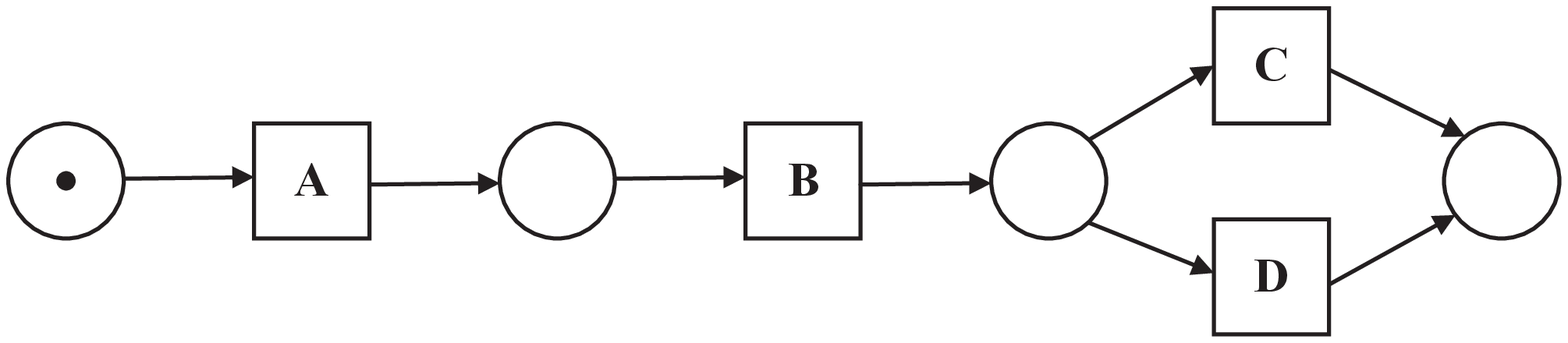}}

		\subfigure[Workflow net generated by SPM-based approaches]{
			\label{fig:subfig:1c} 
			\includegraphics[width=0.7\textwidth]{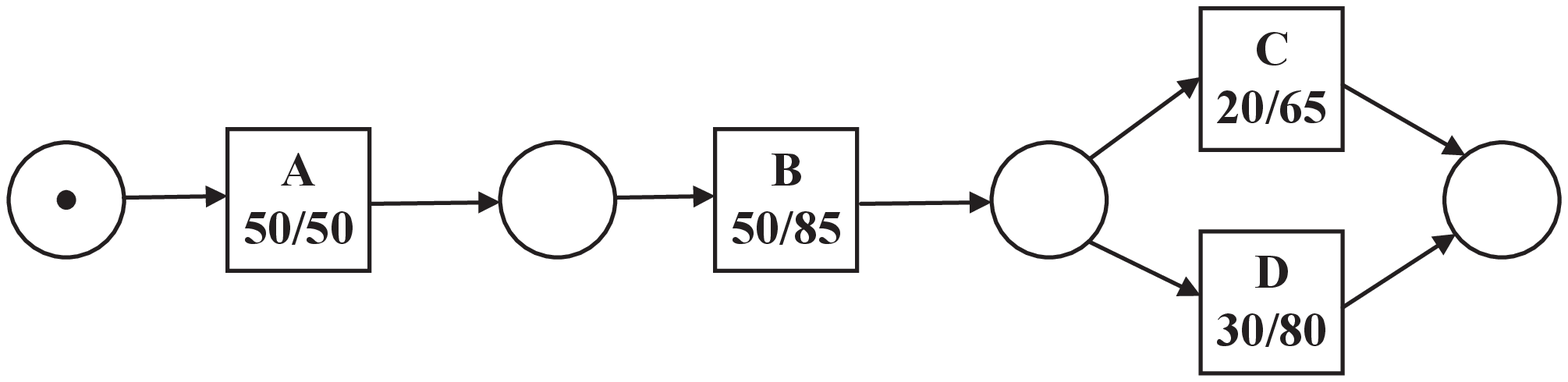}}
	\end{minipage}
	
	\caption{Two main approaches of mining workflow nets from process data}
	\label{1} 
\end{figure}

\subsection{Frequent Subgraph Mining}

Frequent Subgraph Mining (FSM) has been proposed to extract all the frequent subgraphs whose occurrence counts are no less than a specified threshold in a given dataset \cite{cook_substructure_1994}. FSM has practical importance in several applications, ranging from bioinformatics to social network analysis. According to the computation schema, two main types of FSM algorithms are relevant to our work: 

(1) Exact techniques. The most typical exact FSM algorithm is gSpan, which has been used by most researchers for a long time. gSpan adopts a rightmost path expansion strategy to generate candidates \cite{xifeng_yan_gspan:_2002}. Since its effective depth-first search strategy and the DFS lexicographic ordering, the subgraph isomorphism problem, which is NP-Hard, can be solved by comparing the corresponding DFS sequences. However, most existing FSM algorithms are designed for specific datasets, such as biological datasets and molecular datasets \cite{nijssen_quickstart_2004,deshpande_frequent_2005,borgelt_mining_2002}. Some additional constraints, such as closeness, maximal, gap satisfaction, particular structure, and community property, have been added to accelerate the computation \cite{uno2016mining,al-naymat_enumeration_2008,huan_spin:_2004,valluri_margin:_2010,yang_computational_2006}. There have been works improving gSpan \cite{lakshmi_efficient_2013}. For example, Li XT et al. \cite{li_efficient_2007} proposed the GraphGen algorithm, which is an improvement of gSpan. GraphGen reduces the mining complexity through the extension of frequent subtree. The complexity of GraphGen is $O(2^n \cdot n^{2.5}/log n)$.

(2) Inexact techniques. These techniques use heuristic rules to prune the search space efficiently. People usually adopt these techniques when mining large graph datasets \cite{kuramochi_grew_2004,venkatasubramanian_reafum:_2015}. In general, heuristic methods can be used to reduce the execution time, while the generated results might be incomplete. 

The key idea of FSM algorithms is to design a more efficient data representation and search schema by adopting some additional constraints according to their dataset, so that FSM algorithms are widely used to analyze various types of graphs. The advantage of FSM algorithms inspired us to analyze Petri nets more efficiently with the idea of pattern growing. The main difference between the conventional FSM algorithms and our PSpan algorithm is that the structure and semantics of a Petri net is very different from a conventional graph as we mentioned in section \ref{sec:intro}. FSM algorithms cannot be applied to mine frequent subnets of Petri nets directly.

\subsection{Applying FSM to Workflow net Mining}
To the best of our knowledge, there hasn’t been work on mining general Petri nets such as C/E nets directly from the process data. With the development of process mining techniques in Petri net, FSM has been adapted successfully to mining frequent subnets from workflow nets  \cite{greco_mining_2005,greco_mining_2006,chapela-campa_towards_2017,belhajjame_keyword-based_2017,cheong_psm-flow_2018,diamantini_discovering_2015,garijo_mining_2015,tax_n_use_2017}. The idea of applying FSM to workflow subnet mining has been extensively studied. The first such algorithm is the w-find algorithm, which has been widely applied in workflow management systems \cite{greco_mining_2006}. However, they view a workflow net as a directed graph with two kinds of nodes (called workflow graph) and calculate the frequent subgraphs according to the execution paths of the workflow nets \cite{diamantini_discovering_2015}. Chapela et al. \cite{chapela-campa_discovering_2017} used the w-find algorithm and proposed an efficient tool called WoMine \footnote{https://tec.citius.usc.es/processmining/womine/} to retrieve both infrequent and frequent patterns from process models. Garijo et al. \cite{garijo_fragflow_2014} proposed the FragFlow algorithm to analyze workflow nets using graph mining techniques.

For efficiency,  some existing workflow net mining works are based on the SUBDUE algorithm, which is a classical inexact FSM algorithm \cite{cook_substructure_1994}. The key idea of the SUBDUE algorithm is the compression technique, which uses the minimum description length (MDL) as a principle to measure the size of graphs, so that the original graph can be compressed efficiently. Unlike SUBDUE, AGM is a popular exact FSM algorithm that follows the level-wise search strategy, but it is inefficient. 
Table ~\ref{table1} shows several algorithms of mining frequent subnets from workflow nets based on FSM elaborately.

\begin{table}[htbp]
	
	\centering\scriptsize
	\caption{The algorithms of mining frequent subnets in workflow nets based on FSM}
	\label{table1}
	\begin{tabular}{|c|c|c|c|}
		
		\hline
		\textbf{Algorithms}&\textbf{Year}&\textbf{Main idea}&\textbf{The adopted FSM algorithms }\\
 \hline
 w-find \cite{greco_mining_2005,greco_mining_2006}&2005&\tabincell{l}{Transform a large workflow net into \\ workflow graph}&AGM \\\hline
  FragFlow  \cite{garijo_fragflow_2014}&2014&\tabincell{l}{Transform workflow nets into labeled \\ directed acyclic graphs}&\tabincell{c}{SUBDUE, FSG, gSpan} \\\hline

      \tabincell{c}{Keyword-based \\ Search \cite{belhajjame_keyword-based_2017}}&2017&\tabincell{l}{Transform workflow nets into labeled \\ compact graphs}&SUBDUE\\\hline
      WoMine-i \cite{chapela-campa_towards_2017}&2017&\tabincell{l}{Retrieve infrequent patterns in a \\ workflow net} & AGM\\\hline
	\end{tabular}

\end{table}
A conclusion of this section: by reviewing the available literature, we got the impression that, to our best knowledge, we haven’t yet seen works in mining frequent subnets from general Petri nets other than workflow nets, which is just the topic we will develop in this paper. In addition, the subnets found in their works are not necessary workflow nets again. Different from them, our principle followed in this paper is: frequent subnets of X-type Petri nets should be also X-type Petri nets.

\section{Net graph and PSpan algorithm}\label{sec:preliminaries}

\subsection{Net Graph}

To perform the process of mining frequent subnets from pure C/E nets efficiently, we transform the pure C/E nets with two kinds of nodes into another representation form, which is more similar to a planar graph. For that purpose, we propose a novel pseudo-graph representation called net graph (N.G. for short), which has only one kind of node. Each event node $e$ of a pure C/E net is transformed into a node $v$ of N.G., while the information of pure C/E nets’ condition nodes is absorbed in the form of tagging into the nodes and edges of N.G.s. For details, see definition ~\ref{NG}.

\begin{definition}[Net Graph]\label{NG}
	A net graph $NG=(V,D,W)$ is a pseudo-graph transformed from a pure C/E net $N=(C,E;F)$, where\\
	1. $V$ is the set of NG's nodes, $V=E$. $D$ is the set of NG's conditions, $D=C$. In representation, each node $v\in V$ has a tagging containing a sequence of signed conditions, which
	correspond to all condition nodes connected to $v$'s counterpart $e \in E$ of $N$. For any condition $c$ in this sequence, the sign is '-'('+') if $F$ contains an arc from $c(e)$ to $e(c)$. Note that in this sequence, the signed conditions are alphabetically ordered where symbol ‘-’ is always before symbol ‘+’;   \\
	2. $W$ is the set of  NG's edges, where the tagging of each edge $w$ is an ordered sequence of triples $\{({h_u}_i,c_i,{h_v}_i)\}$, where $i \in \{1,2,...,I\}$, where $c_i$ is a condition connecting the two nodes $u_i$ and $v_i$. ${h_u}_i='-'('+')$ means $F$ contains an arc from $c_i(u_i)$ to $u_i(c_i)$. The same for ${h_v}_i$.  $I$  is the number of triples in this sequence, i.e., the number of different conditions connecting $u_i$ and $v_i$. The triples appear in the lexicographical order. Symbol ‘-’ is always before symbol ‘+’.

\end{definition}	

$\hfill \square$

  More exactly, the above definition can be written in algorithm form:

\begin{algorithm}[H]
	\normalem
	\caption{PSpan-Net-to-Net-Graph ($N$)}
	\label{a1}

	\LinesNumbered

	Construct a new net graph $NG$ with its node set $V$ be the event nodes $E$ of the pure C/E net $N$;\\
	\ForEach{node $v\in V$}{Set up its tagging as described in definition ~\ref{NG};}
	 
	\If{event nodes $e_1$ and $e_2$ of $N$ are connected by some condition nodes}{Construct an edge of $NG$ connecting $v_1$ and $v_2$ which correspond to $e_1$ and $e_2$ respectively;}
	\ForEach{edge $e$ of $NG$}{Set up its tagging as described in definition ~\ref{NG};}

\end{algorithm}

\begin{lemma}\label{rule}
	A net graph $NG$ transformed from a net $N$ satisfies the following properties:\\
	1. Each one-sided condition $c$ appears in the tagging of one and only one node of the $NG$, where ‘one-sided’ means it corresponds to a condition node $c$ connected to only one event node $e$ in $N$;\\
	2. A multi-sided condition node may appear in more than one edge’s tagging, but at most once in the tagging of the same edge;\\
	3.	The tagging of a node may be empty. But the tagging of an edge must not be empty;\\
4.	If the same condition $c$ appears in the tagging of more than one edge connecting the same node $v$, then all occurrences of $c$ on these edges should have the same sign in $v$’s direction.

\end{lemma}
\begin{proof}
	Omitted.
\end{proof}
$\hfill \square$

\begin{example}
	Figure ~\ref{fig:subfig:3a} is a pure C/E net, Figure ~\ref{fig:subfig:3b} is the corresponding net graph transformed from the C/E net, where the tagging of the net graph edges shows that $c_9$ is an input condition of both $v_1$ and $v_2$, but an output condition of $e_3$. 

\begin{figure}[H]
	
		\begin{minipage}[c]{.5\linewidth}
			\centering
			\subfigure[A pure C/E net]{
			\label{fig:subfig:3a}
			\includegraphics[width=\textwidth]{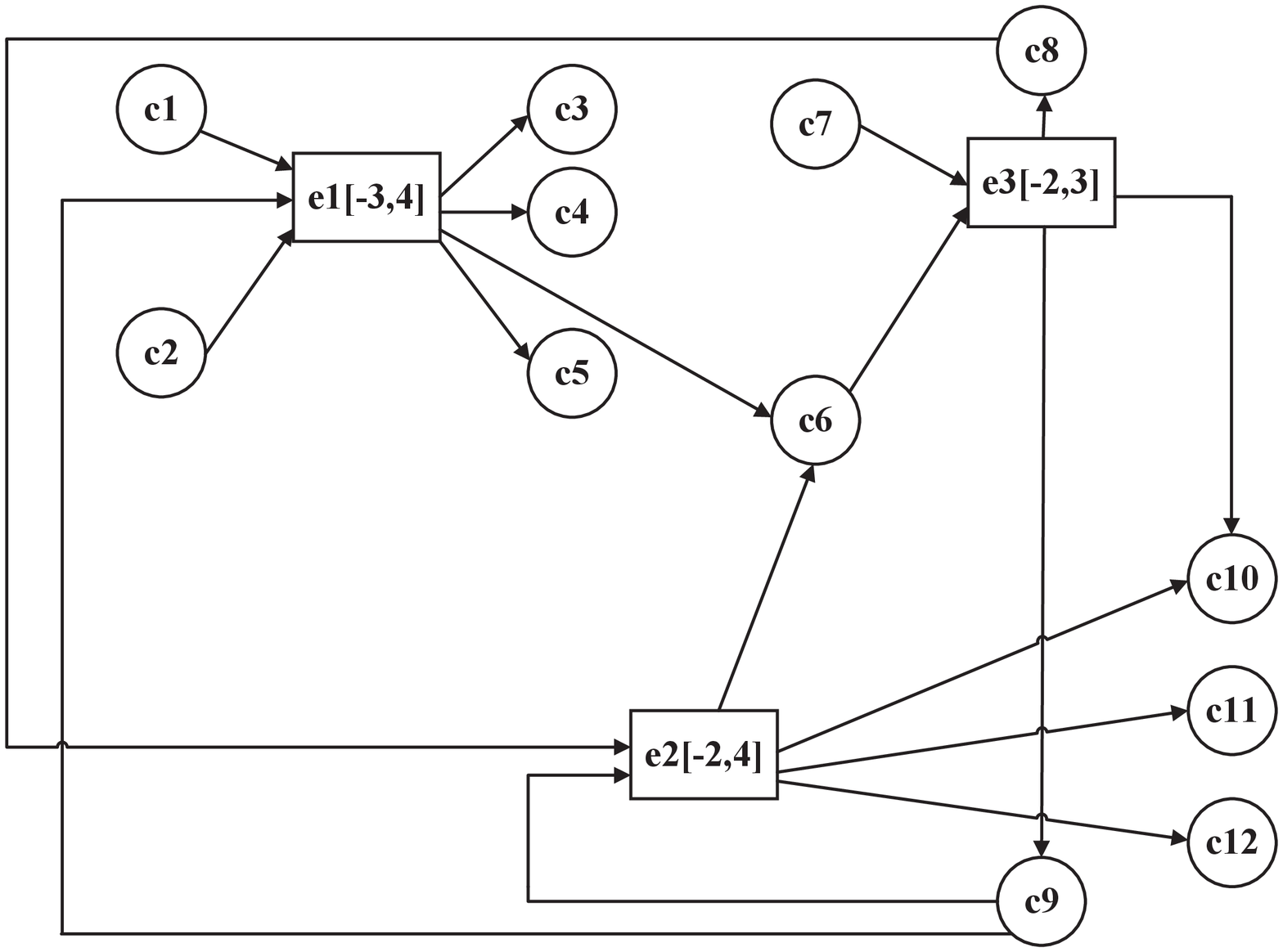}}
	\end{minipage}
	\begin{minipage}[c]{.5\linewidth}
		\centering
		\subfigure[The correponding net graph]{
			\label{fig:subfig:3b} 
			\includegraphics[width=\textwidth]{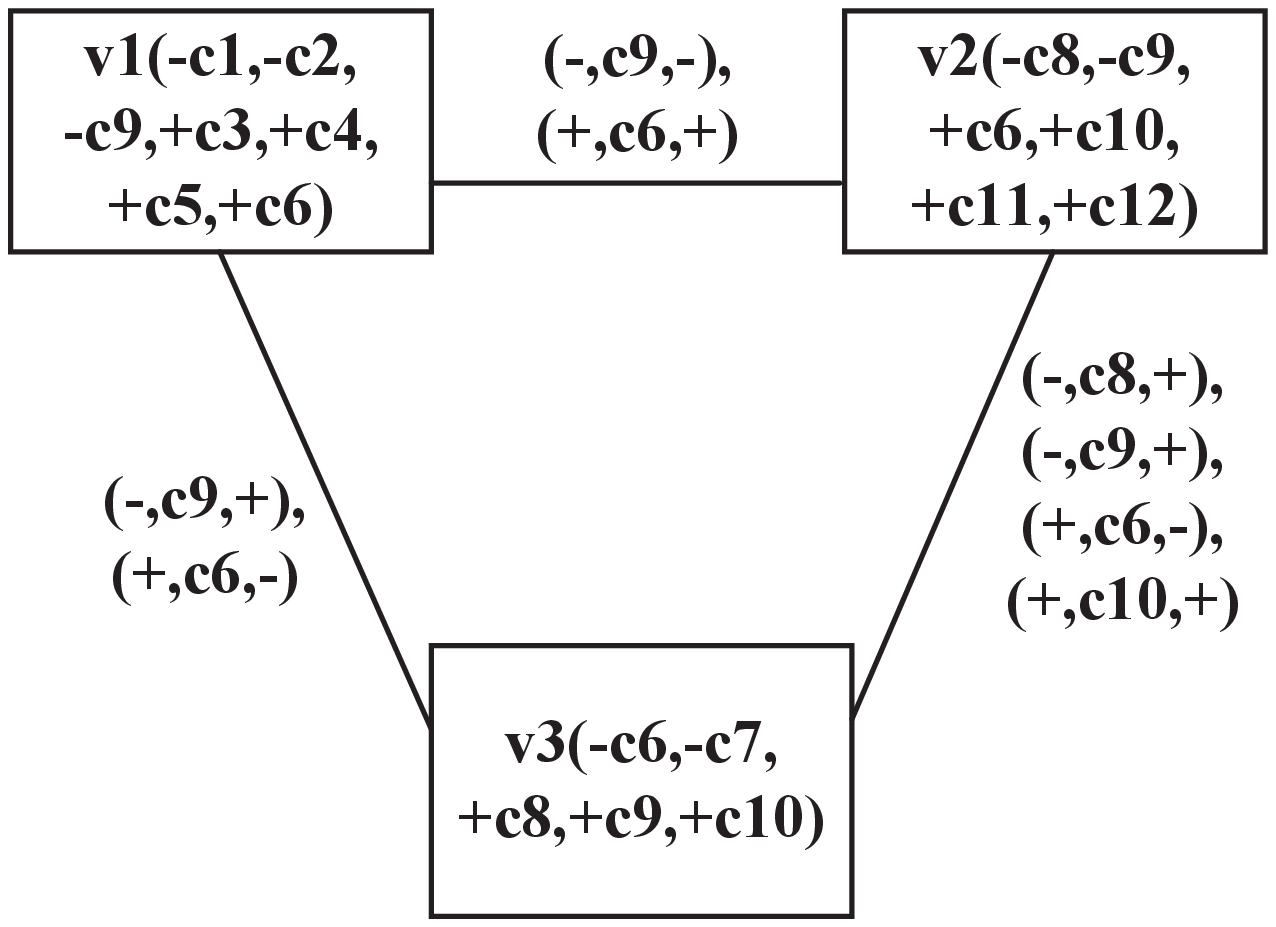}}

	\end{minipage}

\caption{A pure C/E net and its corresponding net graph. }
\end{figure}
\end{example}

Since a single condition node $c$ of a pure C/E net may produce multiple occurrences of condition $c$ in net graph’s tagging, and since different condition nodes in a pure C/E net may have the same name, we have to establish a rule to decide which ones of condition c occurrences in a net graph are transformed from the same condition node $c$ in the original pure C/E net. 

\begin{definition}
An edge of a net graph is a $c$-edge if it’s tagging contains the condition $c$. A $c$-complex is a connected set of $c$-edges. 
\end{definition}
$\hfill \square$

\begin{lemma}
1. All occurrences of a condition $c$ in a $c$-complex correspond to the same condition node in the pure C/E net $N$;\\
2. Any two occurrences of a condition $c$ in two different $c$-complexes correspond to different condition nodes in the pure C/E net $N$.
\end{lemma}
\begin{proof}
The first assertion is true because any two neighbor $c$-edges in the $c$-complex share a common node $v$, who’s counterpart $e$ in $N$ cannot connect the same condition node twice. This means all $c$ occurrences in this $c$-complex correspond to the same condition node $c$ in $N$. Otherwise it would contradict the precondition that $N$ is a pure C/E net and therefore shouldn’t have self-loops. The second assertion is also true since according to algorithm \ref{a1} and definition \ref{NG}, these two occurrences of condition $c$ correspond to different condition nodes in $N$.
\end{proof}
$\hfill \square$

The following algorithm transforms a net graph into a pure C/E net.

\begin{algorithm}[H]
		\normalem
	\caption{PSpan-Net-Graph-to-Net ($NG$)}
		\label{a2}

	\LinesNumbered

Construct a new net $N$ with its event node set $E$ be the nodes $V$ of the net graph $NG$;\\
\If{$c$ is a condition appearing in node $v$'s tagging, but not in the tagging of any edge connecting $v$}{ Attach a condition node $c$ to $v$'s counterpart;\\
\ForEach{event node $e$}{\If{ $c$'s sign in $NG$ is '+'}{Adding an arc in $N$ from $e$ to $c$;}\ElseIf{$c$'s sign in $NG$ is '-'}{Adding an arc in $N$ from $c$ to $e$;}  }

}
 \ForEach{$c$-complex in $NG$}{Set up a condition node $c$ in $N$;\\
 Connect $c$ with each event node $e$ who’s counterpart $v$ is a node of this complex by constructing an arc between them;\\
  The rule for deciding arc’s direction is the same as above;
 	}

\end{algorithm}
\begin{theorem}
	For any pure C/E net $N$,\\
	(1) There exists a transformation $\varphi$ such that for each pure C/E net, $\varphi(N)$ is a net graph.\\
	(2) There exists an inverse transformation ${\varphi}^-1$, such that ${\varphi}^{-1}(\varphi(N))=N$.\\	
\end{theorem}
\begin{proof}
Omitted.
\end{proof}

$\hfill \square$
\subsection{Net Graph Traversal and DFS code} \label{traverse}
In the following, we give a complete list of PSpan sub-algorithm abstracts at first. Motivated by the pattern-growth mining approach, PSpan consists of five steps. (1) Transform pure C/E nets into net graphs (algorithm \ref{a1}); (2) Traverse the net graphs in the minimal DFS manner (algorithm \ref{traversal},\ref{min}); (3) Filter the minimal DFS codes to get frequent edges in net graphs (algorithm \ref{a5}); (4) Mine frequent sub-net graphs (algorithm \ref{a6},\ref{a7}); (5) Transform the results back into frequent complete subnets of pure C/E nets (algorithm ~\ref{a2}). The framework of the PSpan algorithm can be seen in Figure ~\ref{framework}.

\begin{figure}
	\centering
	\includegraphics[width=0.7\textwidth]{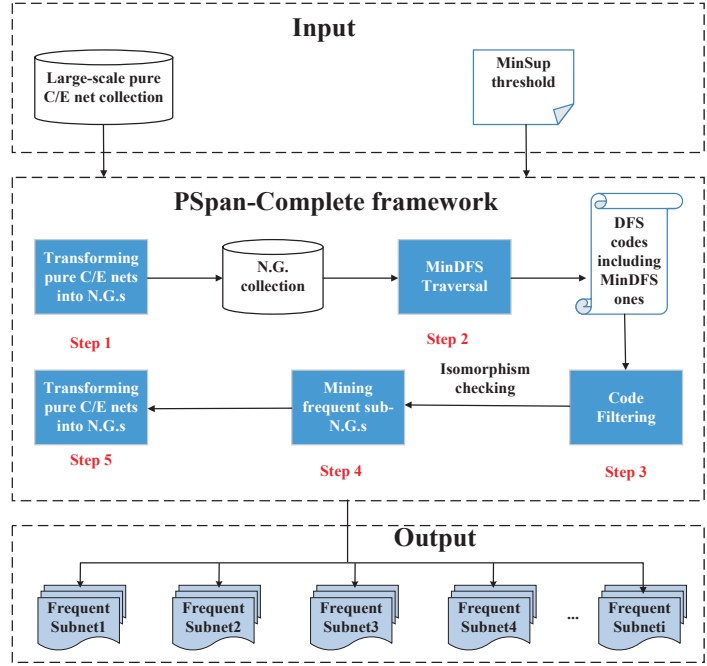}
	
	\caption{The framework of PSpan algorithm. }
	\label{framework}
\end{figure}

In the following algorithm, we use a new strategy of graph traversing. Rather than traverse a net graph in a just depth-first but otherwise arbitrary way, we traverse a net graph in a minimal depth-first way. That means we always select an alphabetically minimal edge for the next depth-first traversal step. This strategy has the advantage that we do not need to sort the DFS codes afterward to get the minimal DFS codes. Let’s take an example. Consider figure \ref{fig:subfig:4a} and calculate the number of possible traversals in the traditional depth first way. For the start node, we may select $v_1, v_2$ or $v_3$. For each of them, there are two possible edges to be selected from. Then the remaining part of traversal is uniquely determined. In total, we have 6 possible traversal paths. But according to our minimal depth first idea, the starting node must be $v_1$, the first edge must be that connecting $v_1$ and $v_2$, because the triple sequence $(-,c_9,-)$, $(+,c_6,+)$ in the edge tagging is alphabetically before the triple sequence $(-, c_9,+), (+,c_6,-)$, since '-' is alphabetically before '+'. As a result, the whole traversal path is uniquely determined.

In the following, we introduce the details of net graph DFS codes. While a DFS code records a traversal of the whole net graph, a DFS code unit records the traversal of an edge of the net graph. As shown by the grammar, each DFS code unit is divided in 6 segments, where the 3-5 segments, i.e., the $<$edge identification$>$ part, are used for frequency determination.

\begin{definition}[DFS code] \label{code}
	The DFS code of a net graph is a depth first travel sequence of code units, where each code unit corresponds to the traversal state of an edge of the net graph. Its syntax formula is as follows. \\
	$<$DFS code$>$::=$<$DFS code unit$>$$|$$<$DFS code unit$>$,$<$DFS code$>$\\
	$<$DFS code unit$>$::=($<$first segment$>$,$<$second segment$>$,$<$edge identification$>$,$<$sixth segment$>$)\\
	$<$first segment$>$::=$<$front node’s traverse order$>$\\
	$<$second segment$>$::=$<$rear node’s traverse order$>$\\
	$<$edge identification$>$::=$<$third segment$>$,$<$fourth segment$>$,$<$fifth segment$>$\\
	$<$third segment$>$::=$<$front node$>$\\
	$<$fourth segment$>$::=($<$edge tagging$>$)\\
		$<$fifth segment$>$::=$<$rear node$>$\\
		$<$sixth segment$>$::=$<$net graph id$>$\\			
	$<$front node$>$::=$<$node name$>$($<$node tagging$>$)\\ 
	$<$rear node$>$::=$<$node name$>$($<$node tagging$>$) \\
	$<$node tagging$>$::= $<$signed condition$>$$|$$<$node tagging$>$, $<$signed condition$>$\\ 
	$<$signed condition$>$::= $<$sign$>$$<$condition$>$\\
	$<$edge tagging$>$::= $<$triple$>$$|$$<$triple$>$, $<$edge tagging$>$\\
	$<$triple$>$ ::= ($<$sign$>$,$<$condition$>$, $<$sign$>$)\\
	$<$sign$>$::= '-' $|$'+'
	
\end{definition}
$\hfill \square$

\begin{example}
The DFS code (consisting of three code units) of figure \ref{NG} produced by minimal depth first search is: \\

(1, 2, \underline{$v_1$ ($-c_1, -c_2, -c_9, +c_3, +c_4, +c_5, +c_6), ((-, c_9,-), (+, c_6, +)),$ $v_2$ ($-c_8, -c_9, +c_6, +c_{10}, +c_{11}, $} \\\underline{$+c_{12}$)},NGid), \\
(2, 3, \underline{$v_2$ ($-c_8, -c_9, +c_6, +c_{10}, +c_{11}, +c_{12}), ((-,c_8,+), (-,c_9, +), (+,c_6, -), (+,c_{10}, +))$, $v_3$ ($-c_6, -c_7, $}\\ \underline{$+c_8, +c_9, +c_{10}$)}, NGid), \\
(3, 1, \underline{$v_3 (-c6, -c7, +c8, +c9, +c10), ((-, c6,+), (+, c9, -))$, $v_1 (-c_1, -c_2, -c_9, +c_3, +c_4, +c_5, +c_6$)}, NGid),\\
where the underlined part is the <edge identification> of each code unit. Note that the edge tagging representation of the third code unit above is different from that of the same code unit in figure \ref{NG}. This is because this edge is traversed from $v_3$ to $v_1$, not from $v_1$ to $v_3$.

\end{example}

\begin{definition}(DFS Lexicographical Order) \label{order}
The sorting rule of DFS codes (only consider the first 5 segments for each code unit) satisfies the lexicographic order defined in the following. This means, for comparing two DFS codes: \\
	1.  Their corresponding code units are compared separately;  \\
2.  The corresponding segments of each DFS code unit are compared separately; \\
3.  For each term of a segment, the corresponding sub-terms are compared separately. \\

\end{definition}
$\hfill \square$

In the following algorithm, we use a queue $SNNG$ to save all visited nodes together with their tagging in traversal order. While traversing the net graph, we call an edge leading from the current node to an unvisited node a forward edge, otherwise a backward edge if it is leading to an already visited node. During traversal, we always choose the minimal edge (determined by the alphabetical order of the code unit representation disregarding its last segment, the net graph id.) starting from the current node. In case there is more than one minimal edge to be chosen from, we chose that edge who’s another end node is minimal. If in the extreme case that the nodes at other ends of the current edges are still the same, then we continue this principle until either a single minimal node or edge is found, or all nodes and edges of the net graph are exhausted (In this case an arbitrary choice is fine).  

\par For frequent subgraph mining, it is enough to mine all frequent edges as the gSpan algorithm does, since frequent nodes alone of a graph are not taken in consideration. However, while mining frequent sub-net graphs, the frequent nodes also matter since each node corresponds to a complete pure C/E subnet. It is possible that a node $u$ is frequent, but the edges connected with $u$ are not frequent. Therefore algorithm \ref{traversal} picks up all frequent nodes of a net graph already during traversing, the frequent edges will be mined only later in algorithm \ref{a5} for code units filtering. 

\begin{algorithm}[H]
	\normalem
	\caption{PSpan-Minimal-DFS-Traversal ($NG$)}
	\label{traversal}
	\KwIn{A net graph $NG$}
	\KwOut{A minimal DFS code of NG in form of a queue of code units}
	\LinesNumbered
	Choose the minimal node $u$ (among NG’s all nodes) as the starting of traversal;\\
 Mark $u$ as current node and already visited, put $u$ and its tagging in $SNNG$;\\
Until all edges traversed do\\
\Begin{
\While{there are untraversed backward edges from current node $u$}{
Call Find-minimal-edge ($u$,0);\\
}
\If{there are untraversed forward edges from current node $u$}{Call Find-minimal-edge ($u$,1);}
}
\end{algorithm}
\begin{algorithm}[H]
		\normalem
	\caption{Find-Minimal-edge ($x,y$)}
	\label{min}
	\If{$y=0(1)$}{\If{there is only one backward (forward) minimal edge $e$ starting from $x$ and leading to some node $v$ or if all next backward (forward) edges starting from $x$ have the same name $e$ and tagging but there’s one who’s another end node $v$ is minimal}{Mark $e$ as traversed and $v$ as visited and is the current node $u$;}\Else{Consider all next backward (forward) edges as only one edge e leading to only one node $v$ and mark $e$ as traversed and $v$ as visited and is the current node $u$;}Put the code unit $e$ in $FDFS$;\\Put the current node $u$ together with its tagging in $SNNG$;}

\end{algorithm}

\begin{lemma}

Given an arbitrary net graph $G=(V,D,W)$, Algorithm \ref{traversal} can generate the minimum DFS code in the minimal DSF traversal way.  

\end{lemma} 

\begin{proof}
	By induction.
\end{proof}
$\hfill \square$
\begin{example}
	Figure ~\ref{fig:subfig:4a} illustrates the net graph $K$. Using the DFS strategy to traverse $K$, different DFS codes can be generated. The bold line denotes the DFS tree. Figure ~\ref{fig:subfig:4b}\ref{fig:subfig:4c}\ref{fig:subfig:4d} are three examples. Forward (backward) edges are represented with bold (dotted) line. Figure ~\ref{fig:subfig:4c} follows the minDFS approach ( Algorithm \ref{traversal}).
\end{example}
\begin{figure}[H]
	
	\begin{minipage}[c]{.4\linewidth}
		\centering
		\subfigure[N.G. $K$]{
			\label{fig:subfig:4a}
			\includegraphics[width=1.2\textwidth]{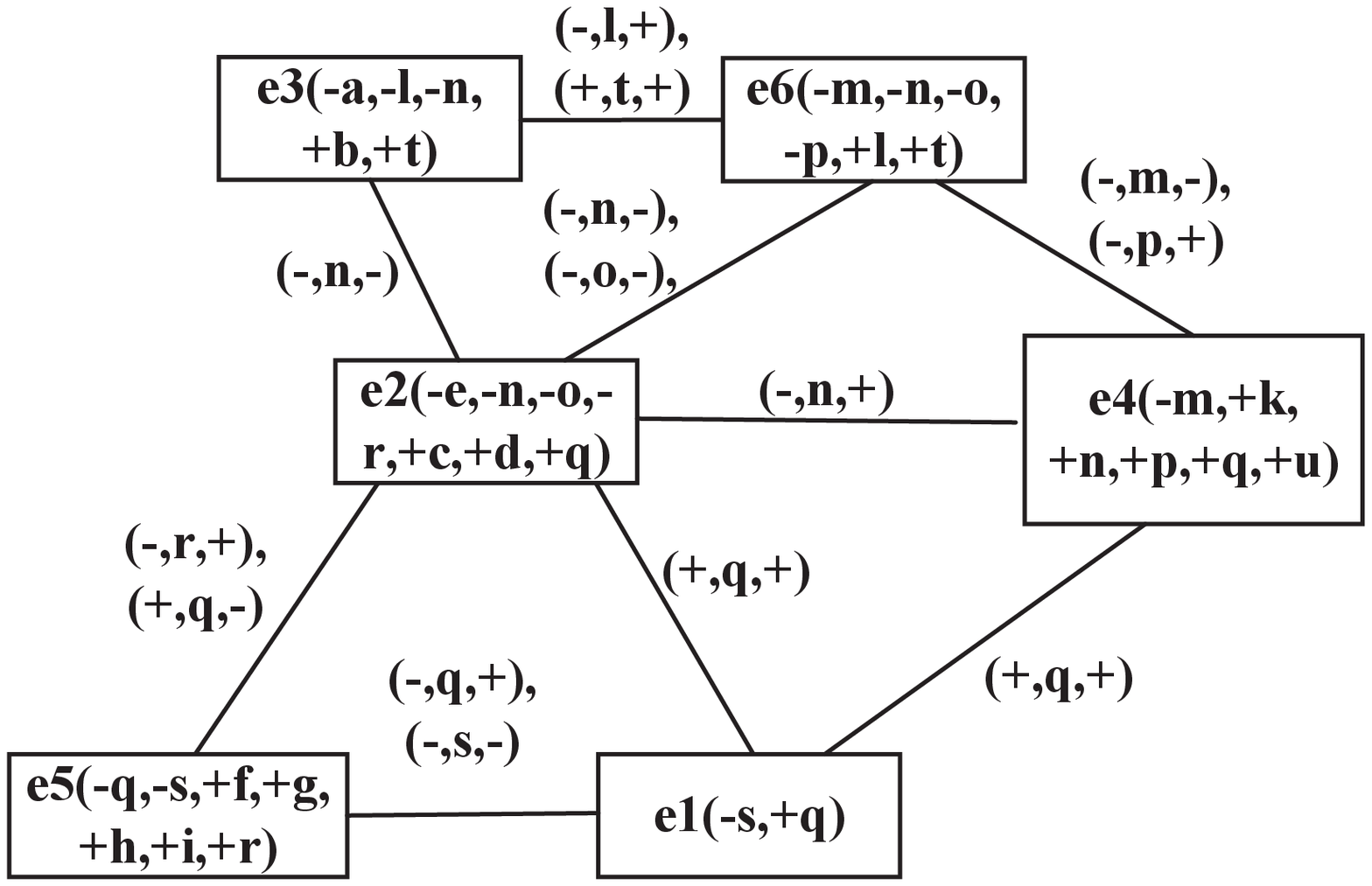}}
	\end{minipage}
	\begin{minipage}[c]{.6\linewidth}
		\centering
		\subfigure[Traversal 1]{
			\label{fig:subfig:4b} 
			\includegraphics[width=0.85\textwidth]{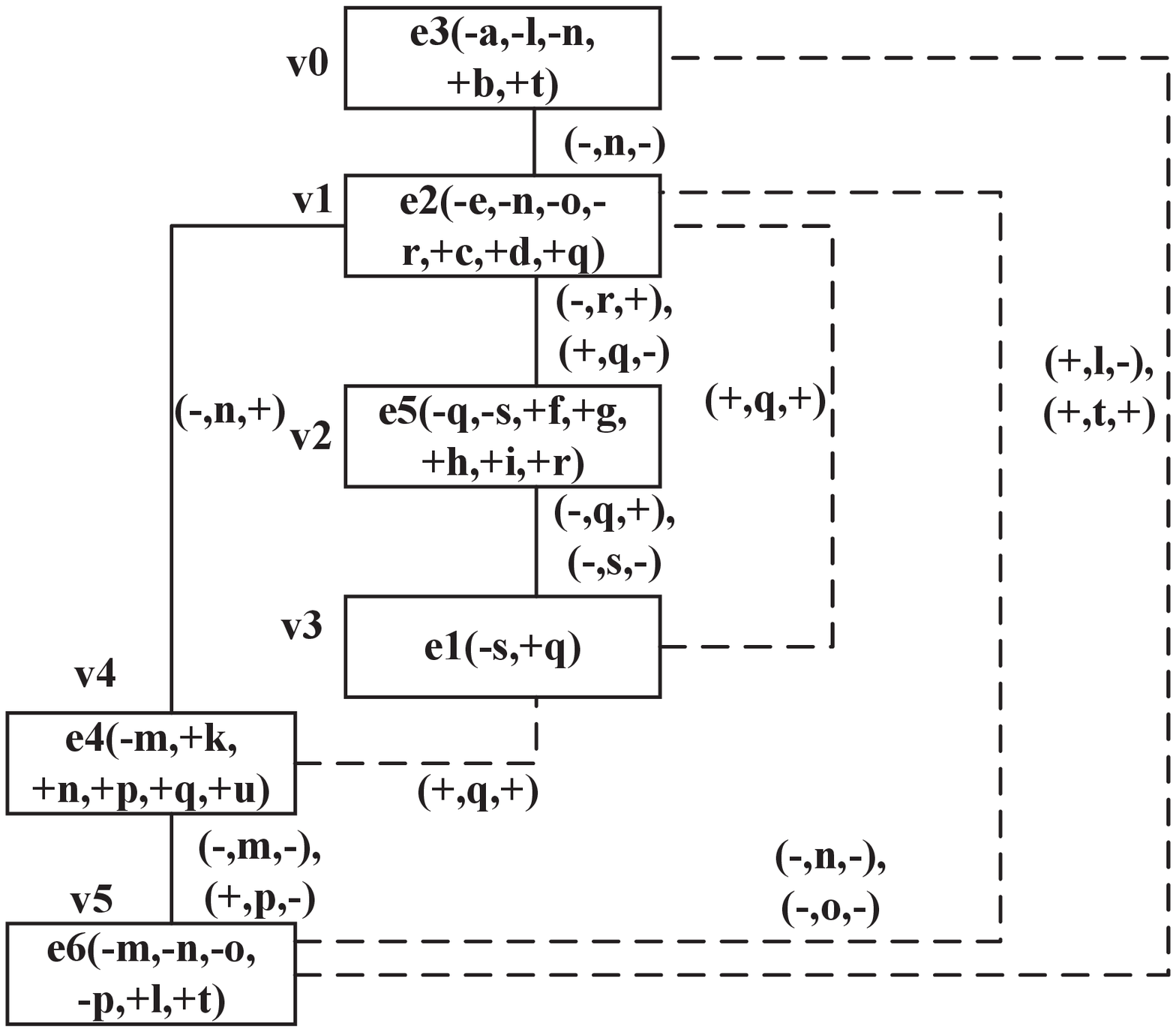}}

	\end{minipage}
		\begin{minipage}[c]{.44\linewidth}
		\centering
		\subfigure[Traversal 2 (minimal)]{
			\label{fig:subfig:4c} 
			\includegraphics[width=\textwidth]{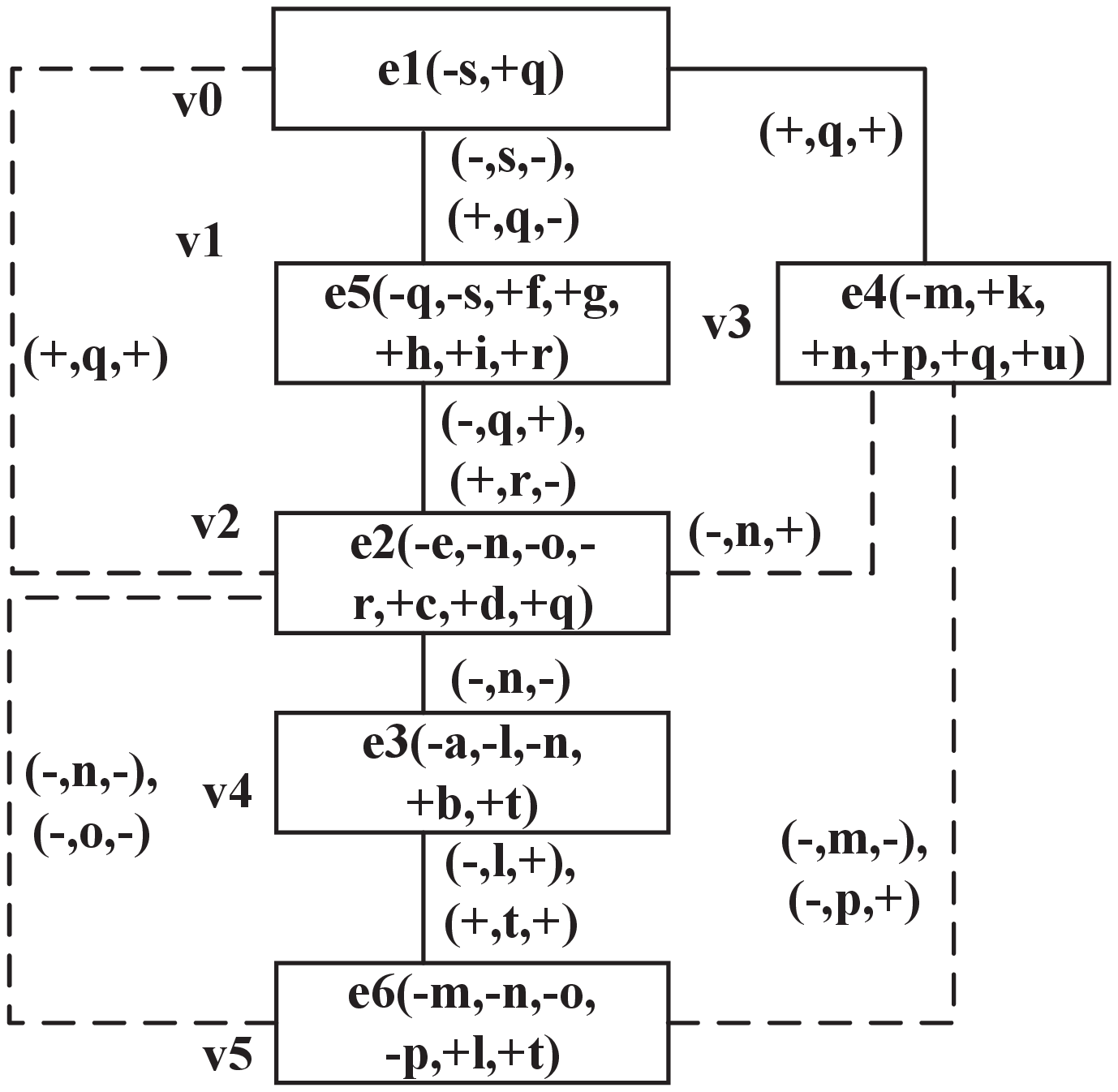}}

	\end{minipage}
	\begin{minipage}[c]{.44\linewidth}
	\centering
	\subfigure[Traversal 3 ]{
		\label{fig:subfig:4d} 
		\includegraphics[width=\textwidth]{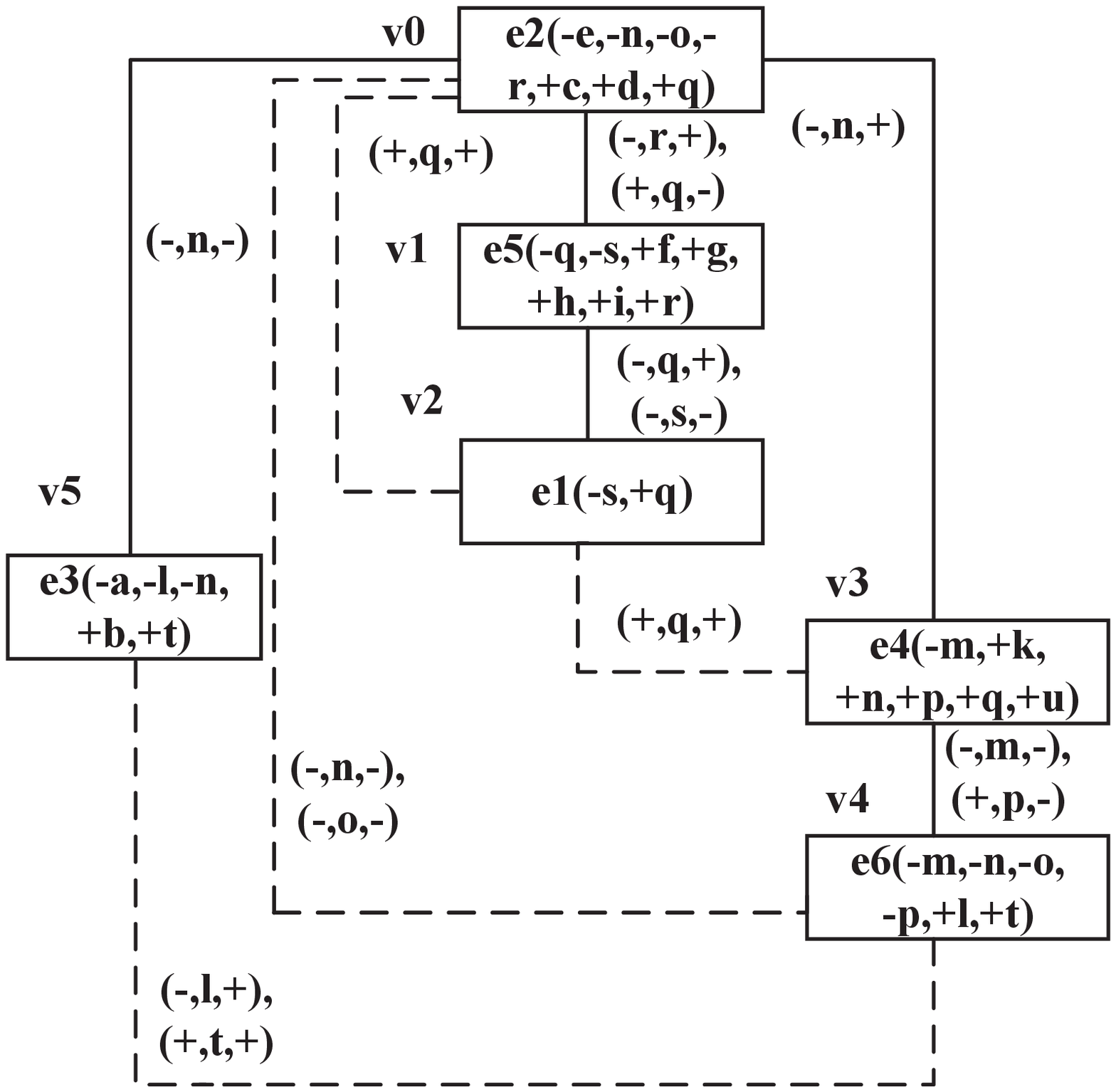}}

\end{minipage}
	\caption{An example of traversing the net graph}
	
\end{figure}

\begin{example}
	Table ~\ref{table2} illustrates the DFS codes of Figure ~\ref{fig:subfig:4b}\ref{fig:subfig:4c}\ref{fig:subfig:4d}, suppose that the id of net graph $K$ is 1. 
\end{example}
\begin{table}[H]
	\centering
		\caption{The DFS codes of Figure ~\ref{fig:subfig:4b}\ref{fig:subfig:4c}\ref{fig:subfig:4d}}
	\label{table2}
	\scalebox{0.8}{\begin{tabular}{|c|c|c|c|}
			\hline
			edge&(b) $\alpha$&(c) $\beta$&(d) $\gamma$\\
			\hline
			0&\tabincell{c}{$(0,1,e_3(-a,-l,-n,+b,+t)$,\\$((-,n,-))
$,\\$e_2(-e,-n,-o,-r,+c,+d,+q),1))$}&\tabincell{c}{$(0,1,e_1(-s,+q),((-,s,-),(+,q,-))$,\\$e_5(-q,-s,+f,+g,+h,+i,+r),1)$}&\tabincell{c}{$(0,1,e_2(-e,-n,-o,-r,+c,+d,+q),$\\$((-,r,+)$,$(+,q,-))$,\\$e_5(-q,-s,+f,+g,+h,+i,+r),1)$}\\
			\hline
		   1&\tabincell{c}{$(1,2,e_2(-e,-n,-o,-r,+c,+d,+q)$,\\$((-,r,+)
				$,(+,q,-)),\\$e_5(-q,-s,+f,+g,+h,+i,+r),1)$}&\tabincell{c}{$(1,2,e_5(-q,-s,+f,+g,+h,+i,+r)$,\\$((-,q,+)$,(+,r,-)),\\$e_2(-e,-n,-o,-r,+c,+d,+q),1)$}&\tabincell{c}{$(1,2,e_5(-q,-s,+f,+g,+h,+i,+r)$,\\$((-,q,+),(-,s,-)),e_1(-s,+q),1)$}\\
			\hline
			2&\tabincell{c}{$2,3,e_5(-q,-s,+f,+g,+h,+i,+r)
	$,\\$((-,q,+),(-,s,-)),e_1(-s,+q),1)$}&\tabincell{c}{$(2,0,e_2(-e,-n,-o,-r,+c,+d,+q)$,\\$((+,q,+)),e_1(-s,+q),1)$}&\tabincell{c}{$(2,0,e_1(-s,+q),((+,q,+))$,\\$e_2(-e,-n,-o,-r,+c,+d,+q),1)$}\\
\hline
	3 &\tabincell{c}{$3,1,e_1(-s,+q),((+,q,+))
	$,\\$e_2(-e,-n,-o,-r,+c,+d,+q),1)$}&\tabincell{c}{$(0,3,e_1(-s,+q),((+,q,+))$,\\$e_4(-m,+k,+n,+p,+q,+u),1)$}&\tabincell{c}{$(0,3,e_2(-e,-n,-o,-r,+c,+d,+q)$,\\((-,n,+)),\\$e_4(-m,+k,+n,+p,+q,+u),1)$}\\
\hline
	4 &\tabincell{c}{$(1,4,e_2(-e,-n,-o,-r,+c,+d,+q)
	$,\\$((-,n,+))$,\\$e_4(-m,+k,+n,+p,+q,+u),1)$}&\tabincell{c}{$(3,2,e_4(-m,+k,+n,+p,+q,+u)$,\\$((-,n,+))$,\\$e_2(-e,-n,-o,-r,+c,+d,+q),1)$}&\tabincell{c}{$(3,2,e_4(-m,+k,+n,+p,+q,+u)$,\\$((+,q,+)),e_1(-s,+q),1)$}\\
\hline	
	5 &\tabincell{c}{$(4,3,e_4(-m,+k,+n,+p,+q,+u)
	$,\\$((+,q,+)),e_1(-s,+q),1)$}&\tabincell{c}{$(2,4,e_2(-e,-n,-o,-r,+c,+d,+q)$,\\$((-,n,-))$,\\$e_3(-a,-l,-n,+b,+t),1)$}&\tabincell{c}{$(3,4,e_4(-m,+k,+n,+p,+q,+u)$,\\$((-,m,-)$,$(+,p,-))$,\\$e_6(-m,-n,-o,-p,+l,+t),1)$}\\
\hline	
	6 &\tabincell{c}{$(4,5,e_4(-m,+k,+n,+p,+q,+u)$,\\$((-,m,-)
	$,$(+,p,-))$,\\$e_6(-m,-n,-o,-p,+l,+t),1)$}&\tabincell{c}{$(4,5,e_3(-a,-l,-n,+b,+t)$,\\$((-,l,+)$,$(+,t,+))$,\\$e_6(-m,-n,-o,-p,+l,+t),1)$}&\tabincell{c}{$(4,0,e_6(-m,-n,-o,-p,+l,+t)$,\\$((-,n,-),(-,o,-))$,\\$e_2(-e,-n,-o,-r,+c,+d,+q),1)$}\\
\hline
	7 &\tabincell{c}{$(5,1,e_6(-m,-n,-o,-p,+l,+t)$,\\$((-,n,-),(-,o,-))
	$,\\$e_2(-e,-n,-o,-r,+c,+d,+q),1)$}&\tabincell{c}{$(5,3,e_6(-m,-n,-o,-p,+l,+t)$,\\$((-,m,-),(-,p,+))$,\\$e_4(-m,+k,+n,+p,+q,+u),1)$}&\tabincell{c}{$(0,5,e_2(-e,-n,-o,-r,+c,+d,+q)$,\\$((-,n,-))$,\\$e_3(-a,-l,-n,+b,+t),1)$}\\
\hline	
	8 &\tabincell{c}{$(5,0,e_6(-m,-n,-o,-p,+l,+t)$,\\$((+,l,-),(+,t,+))
	$,\\$e_3(-a,-l,-n,+b,+t),1)$}&\tabincell{c}{$(5,2,e_6(-m,-n,-o,-p,+l,+t)$,\\$((-,n,-),(-,o,-))$,\\$e_2(-e,-n,-o,-r,+c,+d,+q),1)$}&\tabincell{c}{$(5,4,e_3(-a,-l,-n,+b,+t)$,\\$((-,l,+)$,$(+,t,+))$,\\$e_6(-m,-n,-o,-p,+l,+t),1)$}\\
\hline			
	\end{tabular}}

\end{table}

\subsection{Frequent Subnets Mining}

The following algorithm extracts all frequent edges of a set of net graphs.

\begin{algorithm} [H]
	
	\caption{PSpan-Code-Filtering($FDFS,SNNG,MinSup$)}
	\label{a5}
	\LinesNumbered
	\KwIn{ $FDFS$ and $SNNG$ built in algotrithm \ref{traversal}, the frequency threshold $MinSup$}
	\KwOut{The set $FDFS$ of frequent code units, the set $MinFDFS$ of frequent forward edges, the set $SNNG$ of frequent net graph nodes with tagging}
	
	Sort the code units in $FDFS$ according to the alphabetical order of their $<$edge identification$>$ and $<$net graph id$>$ parts;\\ 
	Remove $FDFS$'s code units who's $<$edge identification$>$ parts are infrequent;\\
	Sort the remaining $FDFS$’s code units alphabetically except the $<$edge identification $>$ parts;\\
	Store those remaining $FDFS$’s code units which denote forward edges also in $MinFDFS$;\\
	Sort $SNNG$'s nodes (with tagging) according to their frequency;\\
	Remove infrequent $SNNG$'s nodes (with tagging);

\end{algorithm}

In the following algorithm, we use an ordered queue $MinFDFS$ to store all DFS code units containing minimal edges from $FDFS$. A minimal edge is a forward edge in the minimal traverse strategy as defined in section 3.2. These edges form the DFS tree in $[]$. The algorithm starts from a code unit $s$ in $MinFDFS$ as its initial code unit when constructing a frequent sub-net graph which grows each time when a new frequent edge from $FDFS$ is added to it. Let $FD(s)$ denote the growing frequent sub-net graph starting from $s$.

\begin{algorithm} [H]
		\normalem
	\caption{PSpan-FqNG-Mining($FDFS,MinFDFS,MinSup$)}
	\label{a6}
	\LinesNumbered

	\KwOut{All frequent sub-net graphs}

\While{$MinFDFS \not = \emptyset$}{
	Take the first code unit $s$ from $MinFDFS$;\\
	\If{$s$ appears in more than $MinSup$ net graphs}{put $s$ in $FD[1]$;\\Call PSpan-FqNG-Construction($s,FDFS,MinSup,2$);}

}	
Put the content of $SNNG$ in $FD[0]$ as additional frequent sub-net graphs (corresponding to frequent 1-complete subnets);

\end{algorithm}

Note that algorithm \ref{a6} only mines frequent sub-net graphs containing at least one edge. As we mentioned before, frequent nodes of net graphs also matter since they correspond to frequent 1-complete sub-pure-C/E nets, subnets for short. All such frequent nodes of net graphs are collected in $FD[0]$ which is the set $SNNG$. See algorithm \ref{min} and algorithm \ref{a5} above. 

\begin{algorithm}[H]
		\normalem

\caption{PSpan-FqNG-Construction($s,FDFS,MinSup,i$)}
\label{a7}
\LinesNumbered
Delete $s$ from $MinFDFS$;\\ \tcc{$s$ contains a code unit $s_0$ and $s_0*e$ is the combination of two code units}
\ForEach{code unit $e$ in $FDFS$}{Form a sub-net graph $s*e$, which means that $s$ is extended with $e$ by combining its $s_0$ with $e$;\\
\If{there is no such $e$}{Return;}
\ForEach{such $s*e$}{\If{$s*e$ appears in more than $MinSup$ net graphs }{Save such $s*e$ in $FD[i]$;\\Remove $e$ from $MinFDFS$ if $e$ is in it;}
Call PSpan-FqNG-Construction($s*e,FDFS,Minsup,i+1$);
}
}

\end{algorithm}

Figure ~\ref{construction} shows the procedures of algorithm PSpan-FqNG-Mining and PSpan-FqNG-Construction, where each $FDFS[i]$ means that part of FDFS contained in $i$-th net graph. Each $FD[j]$ contains all frequent netgraphs consisting of $j$ edges.  

\begin{figure}[H]
	\centering
	\includegraphics[scale=0.3]{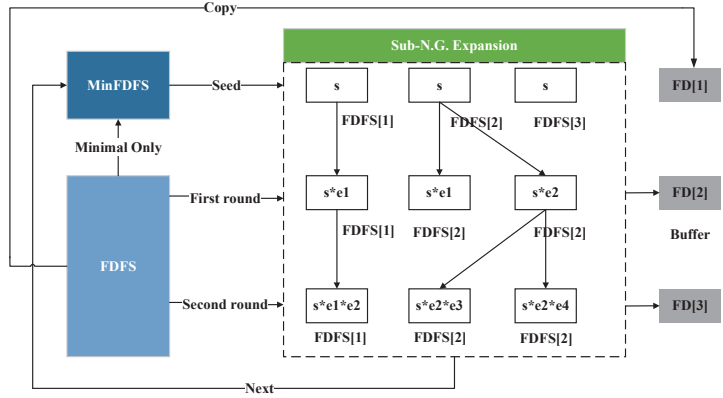}
	
	\caption{The procedure of mining sub-N.G.s in PSpan}
	\label{construction}
\end{figure}

\subsection{Complexity of PSpan}\label{complexity}

Similar as for gSpan, the complexity of PSpan mainly consists of two parts: the complexity of constructing frequent sub-net graphs and that of isomorphism checking. Roughly, the former is $O(2^{n \times m})$ where $n$ is the number of frequent net graph edges, $m$ is the number of conditions. This is because each edge tagging contains a few conditions. This number corresponds to the number of condition nodes connecting a pair of event nodes in a pure C/E net. Usually this number is very small. The complexity of isomorphism checking is also $O(2^{n \times m})$. Therefore the total complexity is $O(2^{n \times m} \cdot 2^{n \times m})$. Since this number is usually small, we can assume it as a constant and keep the complexity  $O(2^{n} \cdot 2^{n})$ as for gSpan. 

What would be the complexity if we were to design an algorithm DSpan for mining frequent pure C/E nets directly on the pure C/E net representation considered as a graph? The complexity would be $O(2^{n} \cdot 2^{n})$  where $n$ is the number of frequent arcs of the pure C/E nets. But here the number $n$ is very large. Assume the pure C/E nets have $J$ frequent event nodes and $K$ frequent condition nodes. The number of possibilities that an event node directs $h$ arcs to $h$ different condition nodes, $1 \le h \le K$ , is $2^K$. On the other hand, the number of possibilities that a condition node directs $m$ arcs to $m$ different event nodes, $1 \le m \le J$, is $2^J$. This means the complexity of such a DSpan would be as high as $O(2^{2^J+2^K} \cdot 2^{2^J+2^K})$. This shows the superiority of our strategy of mining frequent pure C/E subnets on the net graph representation. 

Note that our PSpan algorithm can still be improved by making use of other advanced techniques. For example we can adopt the frequent subtrees technique of GraphGen \cite{li_efficient_2007} to lower down further the complexity of PSpan.

\section{Experimental Evaluation} \label{sec:experiment}
Since large-scale C/E net resources are not available, in this section, we propose a methodology for randomly generating a massive C/E nets reservoir, which contains a series of algorithms. We have implemented the PSpan algorithm and evaluated its performance on the C/E net dataset. To verify the correctness and efficiency of PSpan, we have designed a variant of PSpan, called PSpan2, which serves as a baseline for testing PSpan. The results show that the frequent subnets/subgraphs obtained by the two methods are consistent in the sense of downward inclusion. Our PSpan algorithm outperforms the baseline PSpan2 approach. We implement these two algorithms in C++. All the experiments are conducted on a PC with Intel(R) Core(TM) i7-6700 CPU@3.40GHZ and 32G RAM, and the operating system is Window 10. 
\subsection{A methodology for generating pure C/E net reservoir} \label{methodology}

Algorithms \ref{a9}-\ref{connect} how the procedures of randomly generating a pure C/E net reservoir, where amount is the number of totally generated nets while $U$ is the maximum number of events in a net. $H$ is the maximum number of arcs connecting a generated net and a new subnet added to it. $rU$ and $rH$ are parameters for deciding a random or fixed approach of the net generation. 

\begin{algorithm}[H]
		\normalem
	\caption{ generate-CE-reservoir ($amount,U,H,rU,rH$)}
	\label{a9}
	\LinesNumbered
	\KwIn{Integer $amount>0,U>0,H>0,rU \in \{0,1\},rH \in \{0,1\}$ }

	\For{$i=1$ to $amount$}{Call generate-CE($U,H,rU,rH$) and output the result;\\\tcc{Algorithm \ref{a10},generate a C/E net}}
	
\end{algorithm}

\begin{algorithm}[H]\label{generate}
		\normalem
	\caption{ generate-CE($U,H,rU,rH$)}
	\label{a10}
	\LinesNumbered
	\If{$rH=0$}{$num=U$}\Else{Randomly generate an integer $0 < num \le U$;}

	Initialize $i=1$;\\
	\While{$i \le num$}{
		Call $NX=$1-complete($H,rH$);\\ \tcc{Algorithm \ref{one}, generate a 1-complete C/E net $NX$}
		\If{$i=1$}{$N=NX$}
		\Else{Call connect($N,NX$);}
	
		\tcc{Algorithm ~\ref{connect}, connect the new generated 1-complete with existing nets}
		$i=i+1$;\\
		Return $N$;
	}
	
\end{algorithm}
\begin{algorithm}
		\normalem
	\caption{1-complete($H,rH$)}
	\label{one}
	\LinesNumbered
	Generate an event $X$ randomly \protect \footnotemark[4] \\
	\If{$rH=0$}{$H(X)=H$}\Else{Randomly generate $0 < H(X) \le H$;}
	Generate the number of input and output conditions $X_1 \ge 0,X_2 \ge 0$ randomly, such that $1 \le X_1+X_2 \le H$;\\
	 Generate the set of input (output) conditions $C(C')$ of $X$ such that $|C|(|C'|)=X_1(X_2), C \cap C'=\emptyset$;\\
	Return the generated 1-complete net;
\end{algorithm}
\footnotetext[4] {All the generating algorithms use the Mersenne Twister to generate random numbers. It has been the most classical pseudo-random number generation algorithm proposed by Masamoto and Takusi in 1997. The algorithm has been encapsulated in the standard libraries of languages such as C++ and Python.}
\begin{algorithm}[H]
		\normalem
	\caption{connect($N_1,N_2$)}
	\label{connect}
	\LinesNumbered
	Generate a positive integer $NXC$ randomly that $0<NXC \le min(H(N_1),H(N_2))$;\\\tcc{$H(N_1),H(N_2)$ denote the number of conditions in $N_1$ and $N_2$ respectively}
	$j=1$;\\
	\While{$j \le NXC$}{Choose a condition from each of $N_1$ and $N_2$ randomly, integrate them into one and keep the name of $N_1$’s condition for the combined one;\\
		$j=j+1$;}
	
\end{algorithm}
\begin{lemma}
	All the C/E nets generated by algorithm \ref{a9}-\ref{connect} are pure nets.\\

\end{lemma}

\begin{proof}
	Induction.
\end{proof}
$\hfill \square$
\begin{example}
	Figure ~\ref{random} illustrates a whole procedure of generating a pure C/E net. Suppose given two parameters $U=6,H=8$, which means that each new generated C/E net can contain up to 6 1-complete nets, and each 1-complete net can own up to 8 conditions (including input and output ones). Suppose a C/E net $N_1$ consisting of three 1-complete nets is given in figure ~\ref{random_a}. The random algorithm generates two integers $m=1,n=7$, which means that the new 1-complete net $N_2$ should have 1 event and 7 conditions. The algorithm then randomly divides 7 in 4+3, which means 4 input and 3 output conditions, as figure ~\ref{random_b} shows. $N_1$ and $N_2$ will be then randomly connected by identifying some conditions of $N_1$ with some of $N_2$. All the identified conditions use their names in $N_1$. 
	The connect operation is denoted as $N_1[+]N_2$, and the result after merging is shown in figure ~\ref{random_c}.
\end{example}
\begin{figure}[htbp]
	
	\begin{minipage}[c]{.3\linewidth}
		\centering
		\subfigure[The initial C/E net $N_1$]{
			\label{random_a}
			\includegraphics[width=1.1\textwidth]{pure.eps}}
	\end{minipage}
	\begin{minipage}[c]{.3\linewidth}
		\centering
		\subfigure[The new 1-complete C/E net $N_2$]{ 
			\label{random_b} 
			\includegraphics[scale=0.29]{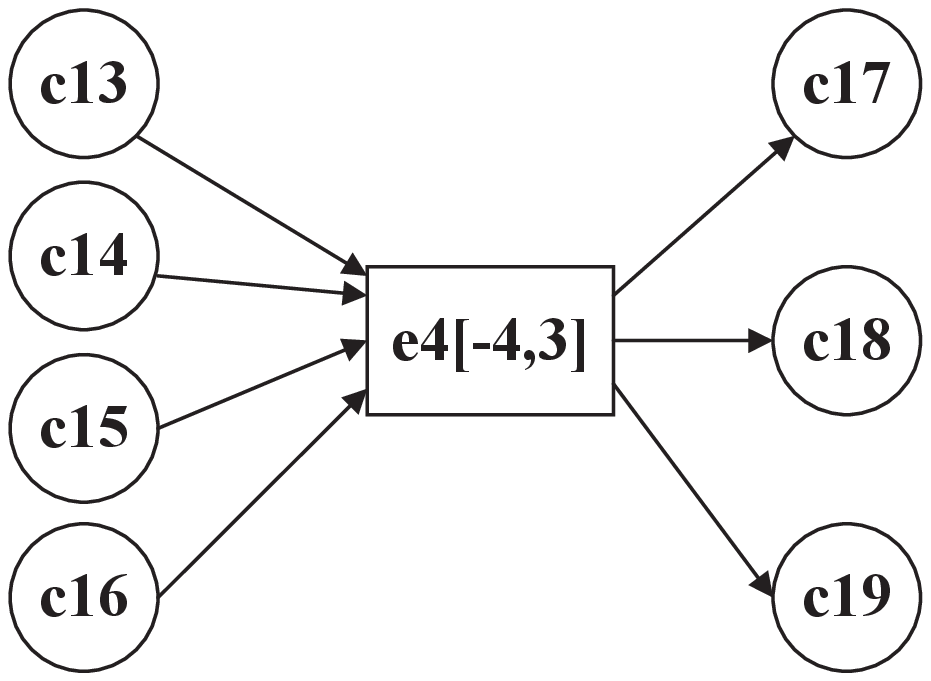}}

	\end{minipage}
	\begin{minipage}[c]{.3\linewidth}
		\centering
		\subfigure[$N_1$ connect with $ N_2$]{
			\label{random_c} 
			\includegraphics[width=1.4\textwidth]{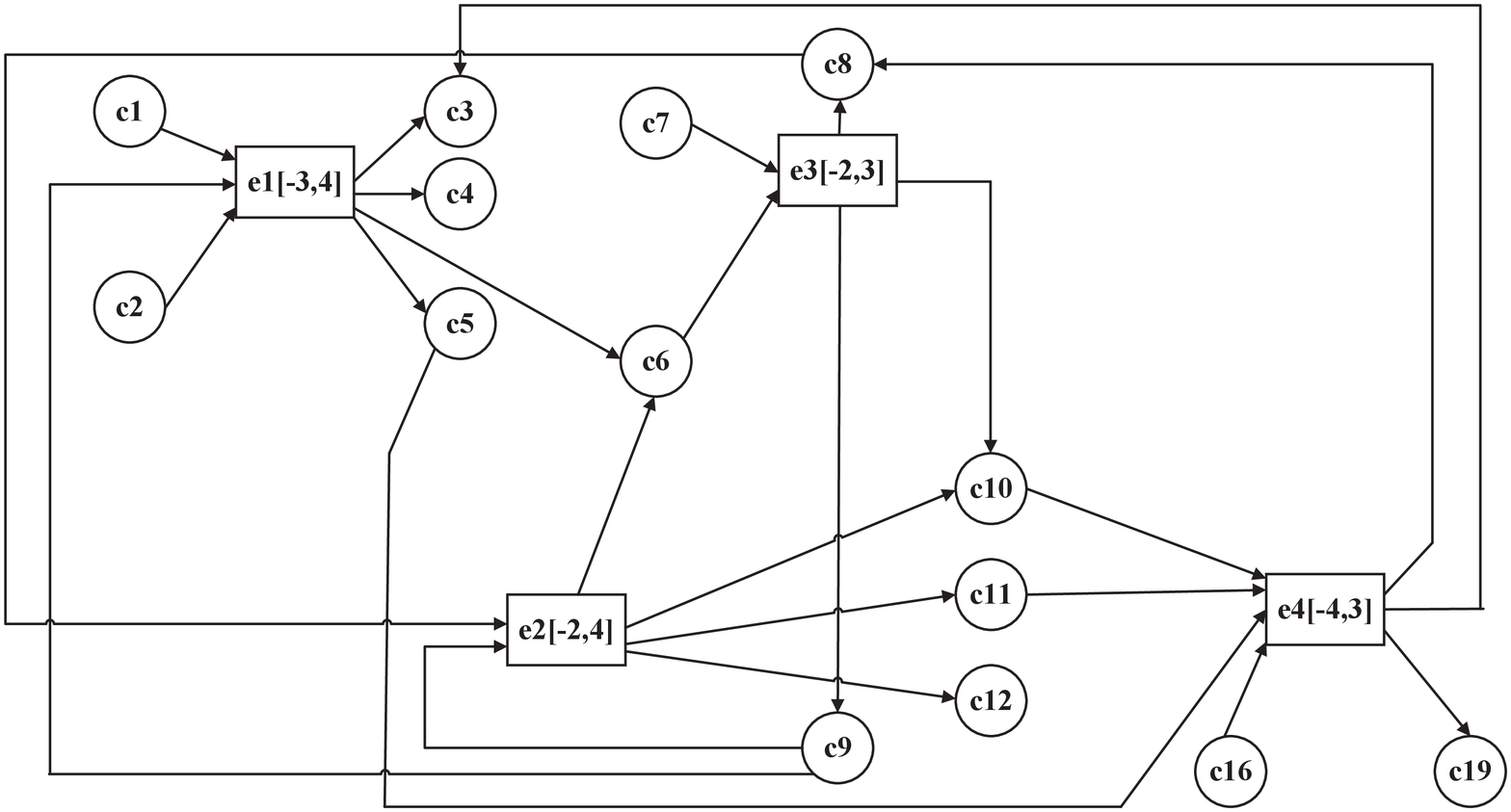}}
		
			\end{minipage}

	\caption{An example of generating a C/E net}
	\label{random}
\end{figure}

\subsection{The Implementation of PSpan algorithm}

We have designed two implementations: (1) PSpan1, which transforms nets to net graphs and then mines frequent subnets on net graphs as introduced in Section \ref{sec:preliminaries}. (2) PSpan2, which transforms nets to net graphs as PSpan1 does, but then transforms these net graphs further into general graphs through symbol transformation, and finally applies the gSpan algorithm to it directly. All the nodes’ labels with their tagging are translated alphabetically into letters with some specific rules for avoiding conflicts. For example, a node with tagging $‘e_2(-e,+c,+d)’$ is represented as $‘e2xuezwczwdy’$, an edge tagging  $‘(-,r,+),(+,q,-)’$ is represented as $‘xuzrzwyzxwzqzuy’$. After the above transformation, the comparison rules of minDFS codes coincide with the DFS lexicographic order in Definition \ref{order}.

In the following, we use the planting method to insert test nets into the nets of the pure C/E net reservoir constructed above. PSpan1 and PSpan2 are performed to mine the frequent subnets, respectively. Our experiments showed that the obtained results of the two approaches are the same.

\subsection{Algorithm testing with planted nets}

In graph mining, the quality of the experimental test data directly affects the performance of algorithms. Recently, the planted motif approach has been widely used in many problems, e.g., sequential pattern mining \cite{lu_exact_2007}, community detection \cite{yang_combining_2009}, and gene co-expression network analyzing \cite{pei_mining_2005}. Many researchers use domain knowledge to construct planted motifs. Inspired by them, we use a random planting method to test our PSpan algorithms. In order to avoid confusion, we use the following terminology: (1) Planting net: net to be planted in other nets. (2) Planted net: net where other nets are already planted in. (3) Test net: net ready for accepting other nets planting in. The details are depicted as follows.

\begin{algorithm}
		\normalem
	\caption{Planting-test-nets ($NS,n,g,MinSup$)}
	 \label{plant}
	\LinesNumbered
\KwIn{A set of test nets $NS$, the number of planting nets $n>0$, two parameters $g,H$, wich denote the maximum of events in each planted net and the maximum of arcs connected with each event, the frequency threshold $Minsup >0$;}

Generate $n$ planting nets randomly;\\\tcc{ Algorithm \ref{a10}}
\ForEach{planting net $x$}{
Generate a random integer $m(x)$ that $MinSup < m(x) \le N$;\\
Choose $m(x)$ test nets randomly, and put them into $S(x)$;\\
\ForEach{selected test net $y \in S(x)$}{Call connect ($y,x$) to produce a planted net $y[+]x$; \\Let $NS(x)=NS\{ where \, all \, test \, nets y \in S(x) replaced \, by \,planted \, nets\, y[+]x\};$}
}

\end{algorithm}

The validation rule is that for each planted net in $NS(x)$ perform PSpan algorithm and check whether the mined results contain the planting net $x$ or not.

\begin{figure}[H]
	\centering
	\includegraphics[scale=0.3]{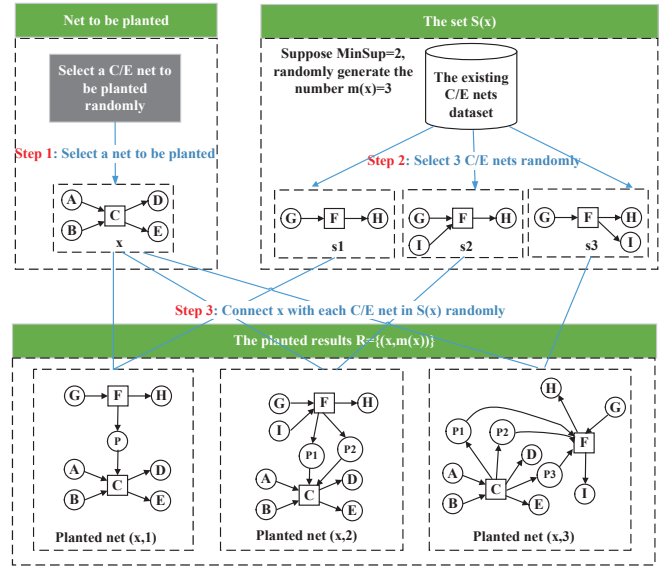}
	
	\caption{The procedure of nets planting}
	\label{valid}
\end{figure}
For example, we generate 10  planting nets randomly in the first step. For each planting net $x$, we choose a subset of test nets from 1000 nets randomly and plant $x$ in all these test nets, where the size of the subset is greater than the MinSup 500. Table ~\ref{results}shows that the number of mined frequent subnets is equal to or slightly bigger than the number of planted nets. This is reasonable since new copies of planting nets may be generated during planting itself.

\begin{table}[htbp]
	\centering
		\caption{The results of planting 10 nets into test nets}
	\label{results}
	\begin{tabular}{|c|c|c|c|c|c|}
		
		\hline
		\tabincell{c}{\textbf{The id of }\\ \textbf{planting net}}&\tabincell{c}{\textbf{\# Nodes} \\\textbf{in the} \\ \textbf{planting}\\ \textbf{net}}&\tabincell{c}{\textbf{\# Arcs}\\\textbf{in the}\\  \textbf{planting}\\\textbf{net}}&\tabincell{c}{\textbf{\#Planted }\\ \textbf{ test nets}}&\tabincell{c}{\textbf{\#Mined}\\ \textbf{freq.}\\\textbf{ subnets by}\\\textbf{ PSpan1/2}}&\tabincell{c}{\textbf{Success}\\\textbf{ratio}}\\\hline
		Planting net 1&11&17&571&571& 1.00\\\hline
		Planting net 2&15&18&506&506&1.0 \\\hline
		
		Planting net 3&15&16&509&509&1.0\\\hline
		Planting net 4&11&19& 595&595&1.0\\\hline
		Planting net 5&15&18& 567&567&1.0\\\hline
		Planting net 6&14&16& 560&560&1.0\\\hline	
		
		Planting net 7&13&15& 576&576&1.0\\\hline	
		Planting net 8&13&17& 519&519&1.0\\\hline
		Planting net 9&9&10& 555&555&1.0\\\hline	
		Planting net 10&12&14& 534&534&1.0\\\hline					
	\end{tabular}

\end{table}

Note that \# means "The number of", similarity hereinafter.\\ Success ratio = \# Mined freq.subnets / \# Planted test nets.

After rigorous verification, we found that PSpan1 and PSpan2 could mine all their planting net copies. Figure \ref{sandian} shows part of the mining results. Take “Planting Net 1” as an example. Due to the space limitation, we only depicted the first 110 of the 571 planted test nets, as shown in figure \ref{sandian}. The horizontal axis represents the planting net 1 occurrences appearing in the planted test nets (1000). For example, $x=20,y=950$ means that the planting net 20 was planted into the test net of id 950. It can be found that the distributions of the results generated by PSpan1 and PSpan2 are the same. 
\begin{figure}[H]
	\centering
	\includegraphics[scale=0.8]{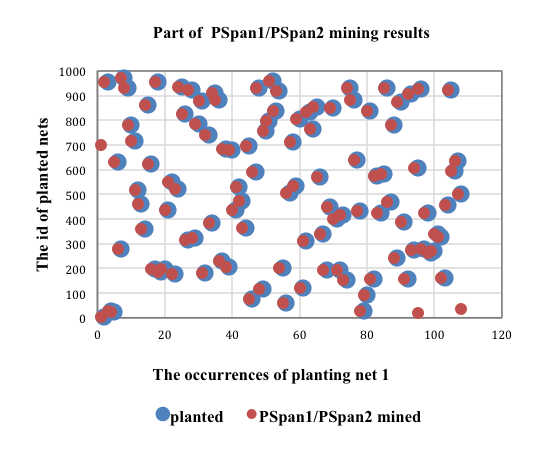}
	
	\caption{Planted test nets and the mined results of PSpan1/PSpan2}
	\label{sandian}
\end{figure}

\subsection{Performance of PSpan}
In Section \ref{complexity},we analyzed the complexity of the PSpan algorithm, and it reveals that the complexities of PSpan1 and PSpan2 are almost the same. Now we make an ideal experiment to show the necessity of introducing net graph representation. We will see what would happen if we considered the pure C/E nets as directed graphs and then apply the gSpan algorithm directly on their net representation (we call this hypothetic approach as DSpan). Our analysis showed that the complexity of DSpan is exponentially higher than PSpan. In this section, we will verify this result with two experiments. (1) The reduction power of net graph structures. When the number of arcs of nets increases, the ratio of net graph edges’ number / net arcs’ number reduces exponentially. (2) The scalability of PSpan1 vs. DSpan. Given the frequency threshold and the number of test nets, the experiments compare the time and space overheads required by PSpan1 and DSpan when the numbers of arcs in the test nets are increasing.

\subsubsection{The reduction power of net graph structures}

We investigate the change of the ratio (net graph edges number / net arcs number) by increasing the number of nodes in test nets (32, 71, 142, 278, 617, including events, conditions) and increasing the number of arcs as shown in Table 4.4.1. The reduction rates can be seen in table ~\ref{tab:c}. It is also depicted in figure ~\ref{figure13}.

Figure ~\ref{figure13} shows that with the rapid increasing of arcs number in test nets, the edges number of net graphs increases only slowly. The compression power of PSpan1 is obvious. 
\begin{table*}[!htb]
	\centering
	\caption{\label{tab:c} The edges number of N.G.s v.s. the arcs number of C/E nets  }
	 \scalebox{0.7}{
		\begin{tabular}{|c|c|c|c|c|c|}
		
		\hline
		
		Avg. arcs number (ARN) in each test net &50&100&200&500&1000\\\hline
		Avg. edges number (AEN) in each N.G.&31&52&70&114&139 \\\hline
		
		\tabincell{c}{AEN/ARN }&62\%&52\%&35\%&23\%&\textbf{14\%}\\\hline
					
	\end{tabular}
		
	}
\end{table*}

\begin{figure}[H]
	\centering
	\includegraphics[scale=1]{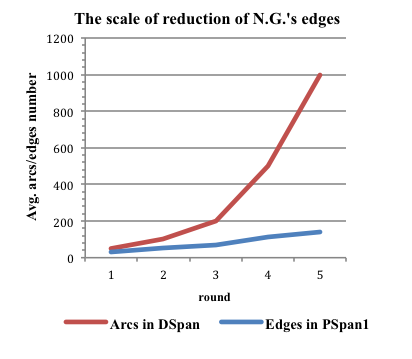}
	
	\caption{The reduction rate of N.G.s’ edge numbers with Petri nets’ arc numbers}
	\label{figure13}
\end{figure}

\subsubsection{The Scalability of PSpan1 vs. DSpan }
The following experiments compare the scalability of PSpan1 vs. DSpan by checking their runtime and space overheads depending on the growing number of C/E net arcs, with a fixed number of test nets (1000) C/E nodes but a varying number of arcs (50,100,150,200,250), and a frequency threshold of 100. At first, we set the frequency threshold as 100, fix the number of test nets as 1000 and the maximum number of events in each planting net as 5. We then compare the running overheads of DSpan and PSpan1 for processing test nets with the number of nodes (35, 63, 92, 122, 147) and the number of arcs (50, 100, 150, 200, 250). The runtime and memory usage can be seen in figure ~\ref{figure15:a} and \ref{figure15:b}. These results coincide with the time complexity analysis in Section \ref{complexity}. They show that the overheads of PSpan1 are significantly reduced compared with those of DSpan. Figure \ref{figure15:a} shows that the growth tendencies of PSpan1 and DSpan have a huge difference. When the number of arcs in a C/E net is growing from 150 to 200 and even larger, PSpan1’s overhead is only several seconds, whereas DSpan has been running for more than 1 day (it cannot be shown in the figure). Figure \ref{figure15:b} shows that memory usage of PSpan1 is much smaller than that of DSpan.  Figure ~\ref{figure16} compares the overheads of PSpan1 and PSpan2.

\begin{figure}[h]
	\centering
	\begin{minipage}{.45\linewidth}
		
		\subfigure[The runtime of PSpan1 and DSpan]{
			\includegraphics[scale=0.9]{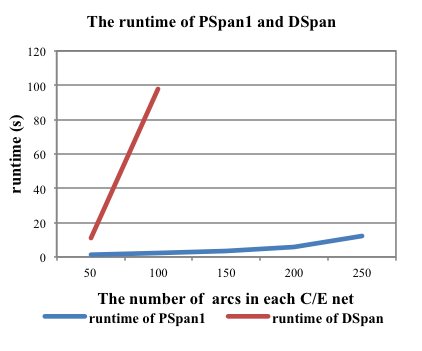}
		\label{figure15:a}}
	\end{minipage}
	\begin{minipage}{.45\linewidth}
		
		\subfigure[The memory used by PSpan1 and DSpan]{
			\includegraphics[scale=0.9]{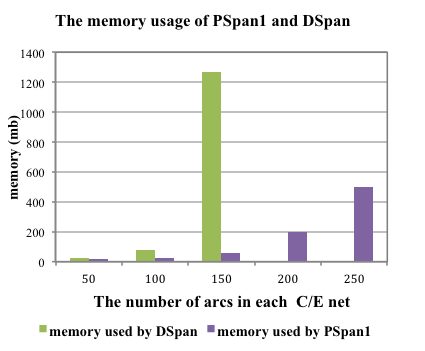}
			\label{figure15:b}}
	\end{minipage}

	\caption{The comparison of PSpan1 and DSpan}
	\label{figure15}
\end{figure}

 \begin{figure}[H]
 	\centering
 	\begin{minipage}{.45\linewidth}
 		
 		\subfigure[The runtime of PSpan1 and PSpan2]{
 			\includegraphics[scale=0.9]{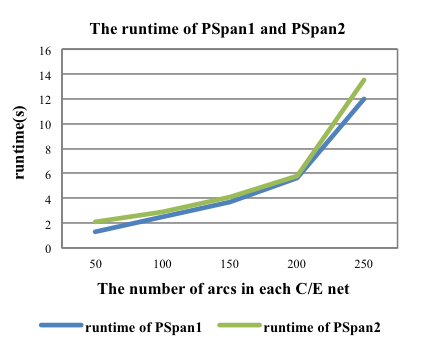}
 			\label{figure16:a}}
 	\end{minipage}
 	\begin{minipage}{.45\linewidth}
 		
 		\subfigure[The memory used by PSpan1 and PSpan2]{
 			\includegraphics[scale=0.9]{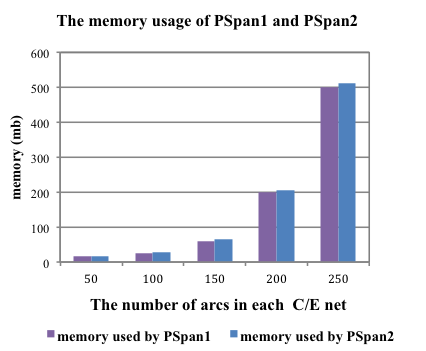}
 			\label{figure16:b}}
 	\end{minipage}

 	\caption{The comparison of PSpan1 and PSpan2}
 	\label{figure16}
 \end{figure}

  Figure \ref{figure16:a} shows that with the number of arcs increasing, both PSpan1 and PSpan2’s runtimes show an exponential growth, but they are in the same order of magnitude. The efficiency of PSpan1 is slightly higher than that of PSpan2. Figure \ref{figure16:b} shows that the memory usage of PSpan1 and PSpan2 are also at the same magnitude. However, the memory used by PSpan1 is slightly less than that of PSpan2.

\section{Extensions to other Subclasses of Petri Nets} \label{sec:extension}

The idea of the PSpan algorithm can be easily extended to other subclasses of Petri nets. Different subclasses of Petri nets have different syntax and/or semantics \cite{peterson_petri_1977}.  Their frequent complete subnets mining algorithms should consider these additional rules. The relevant issues increased the difficulty of designing such algorithms. In the following, we will show how our method can be modified to be applied to eight other subclasses of Petri nets.   
\begin{itemize}
	\item Place/Transition Nets \cite{reisig_lectures_1998}
\end{itemize}

P/T net for short. A P/T net can be represented as a quintuple $N=(S,T;F,K,W)$, where $S,T,F$ denote sets of places, transitions and arcs respectively. $K$ is the set of capacities of $S$ elements, where a capacity denotes the admitted number of tokens in a place. It may be a positive integer or infinity. Note that two separate places with the same name in a P/T net may have different capacities. $W$ denotes the set of arc weights. The weight of an arc shows how many tokens should go through the arc during each firing. A complete subnet of a P/T net should be also a P/T net. Accordingly, the definition of a net graph needs also a modification. We should add weights to the signs in each triple of each edge’s tagging, one after each sign. 

We should also add weights to the signs of one-sided places in each node’s tagging. The P/T net in figure ~\ref{pt} originates from the C/E net in figure ~\ref{fig:subfig:3a} by adding weights to the signs at appropriate positions as described above. Figure \ref{ptng} shows the net graph transformed from figure ~\ref{pt}, where node $t_2$’s tagging in figure ~\ref{pt} is transformed to $t_2(-s_8(K_8,W_{13}),-s_9(K_9,W_{14}),$\\$+s_6(K_6,W_{18}),+s_{10}(K_{11},W_{15}),+s_{11}(K_{11},W_{16}),+s_{12}(K_{12},W_{17}))$, while the edge’s tagging between $t_1$ and $t_2$ becomes $((-W_2,s_9(K_9),-W_{14}),(+W_6,s_6(K_6),+W_{18}))$. Note that the labels of places cannot be decided by their capacity merely because places with the same label maybe have the different capacities in a P/T net.

\begin{figure}[H]
	
	\begin{minipage}[c]{.5\linewidth}
		\centering
		\subfigure[An example of P/T net]{
			\label{pt}
			\includegraphics[width=\textwidth]{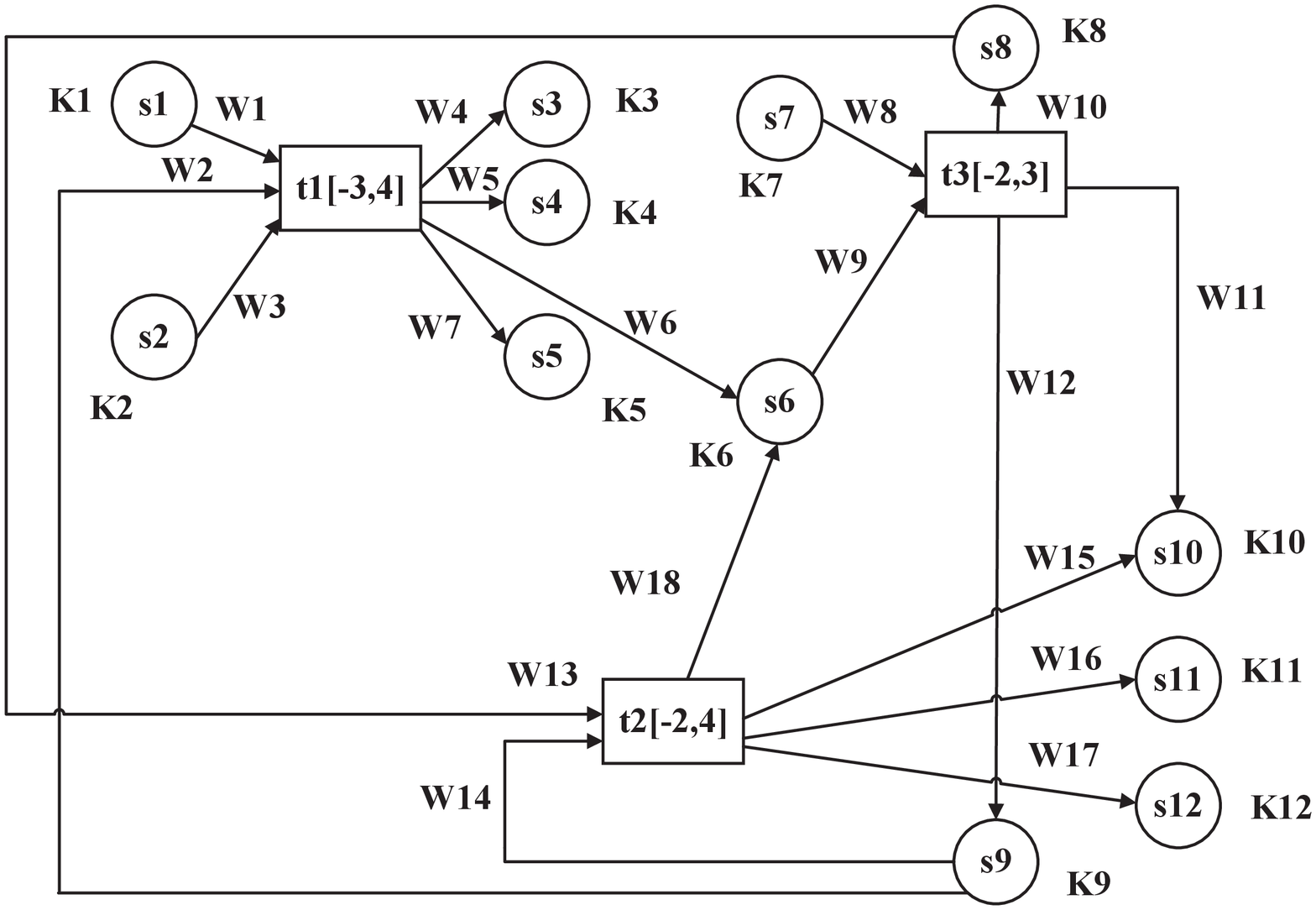}}
	\end{minipage}
	\begin{minipage}[c]{.5\linewidth}
		\centering
		\subfigure[The correponding net graph]{
			\label{ptng} 
			\includegraphics[width=\textwidth]{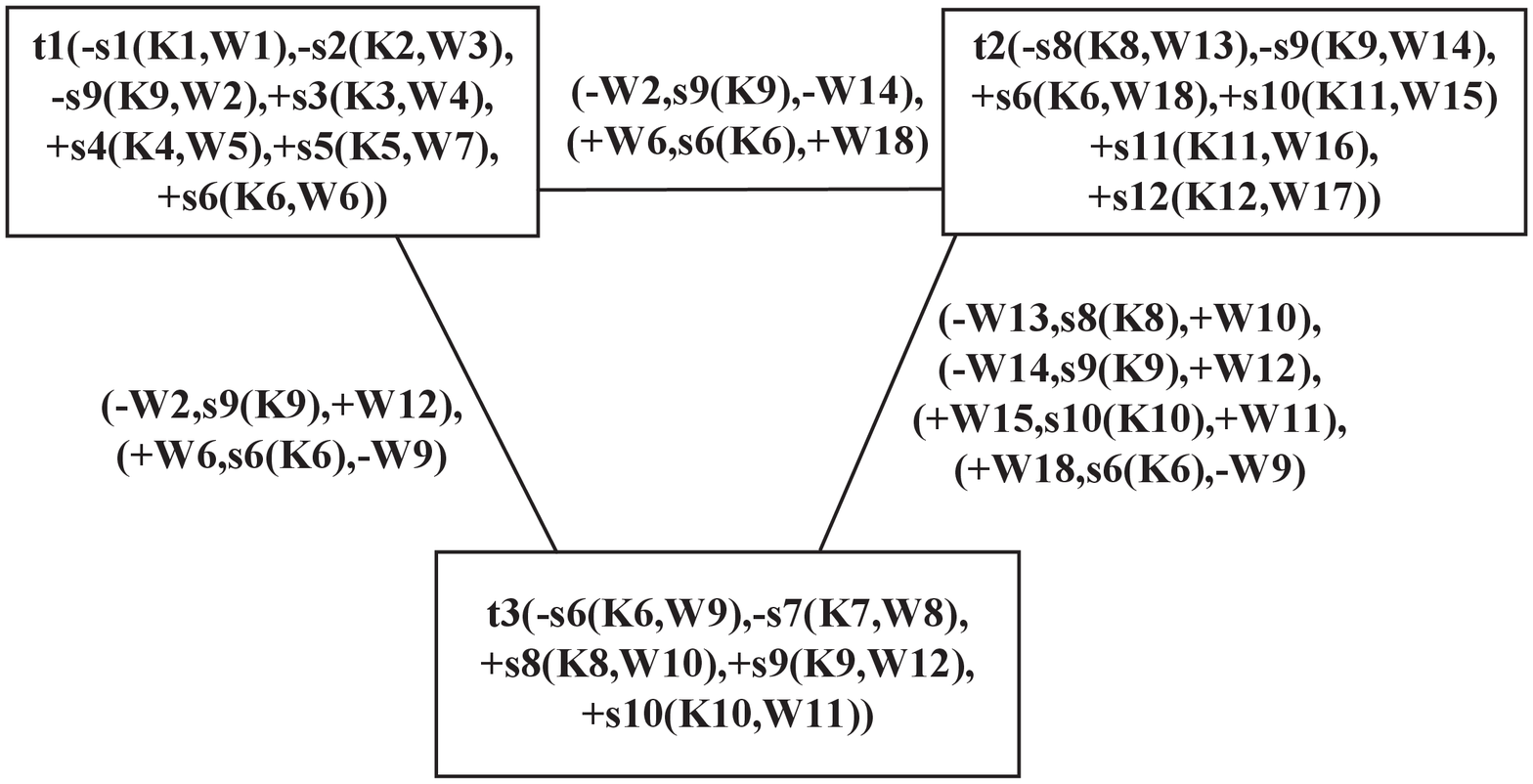}}

	\end{minipage}
	
	\caption{P/T net and its corresponding net graph}
	
\end{figure}

\begin{itemize}
	\item S\_graphs \cite{wu_introduction_2006}
\end{itemize}

For the moment we denote a C/E net with $(S,T;F)$. It is a S\_graph iff $\forall t \in T, |\bullet t|=|t \bullet|=1$. An S\_graph is actually a directed graph and can be processed with frequent subgraph mining algorithms. It can also be processed with our net graph techniques. It is easy to see that a complete subnet of an S\_graph is again an S\_graph. Thus our PSpan algorithm is applicable without any modification.  
\begin{itemize}
	\item Weighted S\_graphs \cite{wu_introduction_2006}
\end{itemize}
S\_graphs where weights are defined on the arcs. A complete subnet of a weighted S-graph is again a weighted S\_graph. PSpan algorithm is applicable without any modification. 

\begin{itemize}
	\item T\_graphs \cite{wu_introduction_2006}
\end{itemize}

A C/E net $(S,T;F)$ is a T\_graphs iff $\forall s \in S, |\bullet s|=|s \bullet|=1$. Like an S\_graph, a T\_graph is also a directed graph. Besides, a complete subnet of a T\_graph is again a T\_graph. Thus our PSpan algorithm is applicable without any modification.   

\begin{itemize}
	\item Weighted T\_graphs \cite{wu_introduction_2006}
\end{itemize}
T\_graphs where weights are defined on the arcs. A complete subnet of a weighted T\_graph is again a weighted T\_graph. PSpan algorithm is applicable without any modification.

\begin{itemize}
	\item Free-choice Petri Nets \cite{hack_analysis_1972}
\end{itemize}

 A C/E net $(S,T;F)$ is a free-choice net iff $\forall t_1,t_2 \in T,t_1 \not = t_2: \bullet t_1 \cap \bullet t_2 \not = \emptyset \to |\bullet t_1|=|\bullet t_2|=1$.  Since any complete subnets of free-choice Petri nets are also free-choice nets, and there is no attached to their nodes or edges, the PSpan algorithm can be used directly.
\begin{itemize}
	\item Occurrence Nets \cite{reisig_lectures_1998}
\end{itemize}

Occurrence nets have been introduced as cycle-free nets with un-branched conditions \cite{reisig_lectures_1998},i.e.,  $\forall s_1,s_2 \in S,|\bullet s_1| \le 1 \wedge |\bullet s_2| \le 1$. Complete subnets of occurrence nets are also occurrence nets. 

An example of S\_graph, T\_graph, free-choice Petri net and occurrence net can be seen in figure ~\ref{16}. Figure \ref{fig:subfig:16a} is a S\_graph, not a T\_graph, figure \ref{fig:subfig:16b} is a T\_graph, not a S\_graph, figure \ref{fig:subfig:16a}\ref{fig:subfig:16b}\ref{fig:subfig:16c} are free-choice Petri nets, both \ref{fig:subfig:16b} and \ref{fig:subfig:16c} are occurrence nets, \ref{fig:subfig:16d} is not any more of S\_graph, T\_graph, free-choice Petri net and occurrence net.

\begin{figure}[htbp]
	
	\begin{minipage}[c]{.24\linewidth}
		\centering
		\subfigure[]{
			\label{fig:subfig:16a}
			\includegraphics[width=0.5\textwidth]{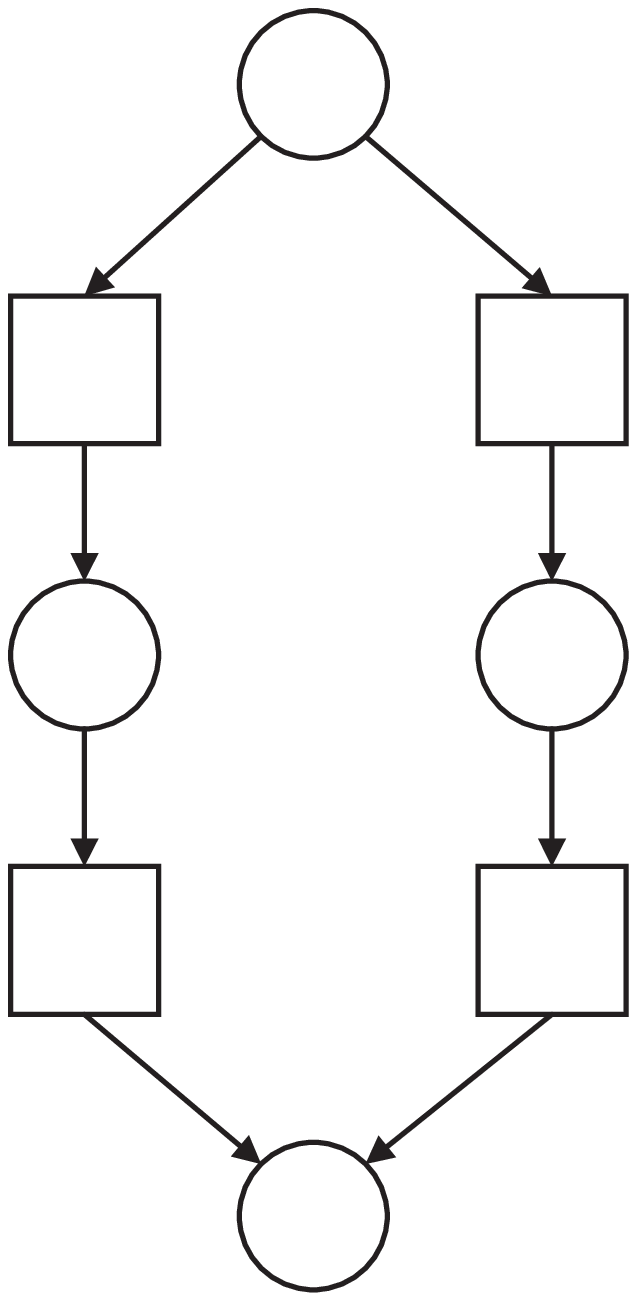}}
	\end{minipage}
	\begin{minipage}[c]{.24\linewidth}
		\centering
		\subfigure[]{ 
			\label{fig:subfig:16b} 
			\includegraphics[width=0.5\textwidth]{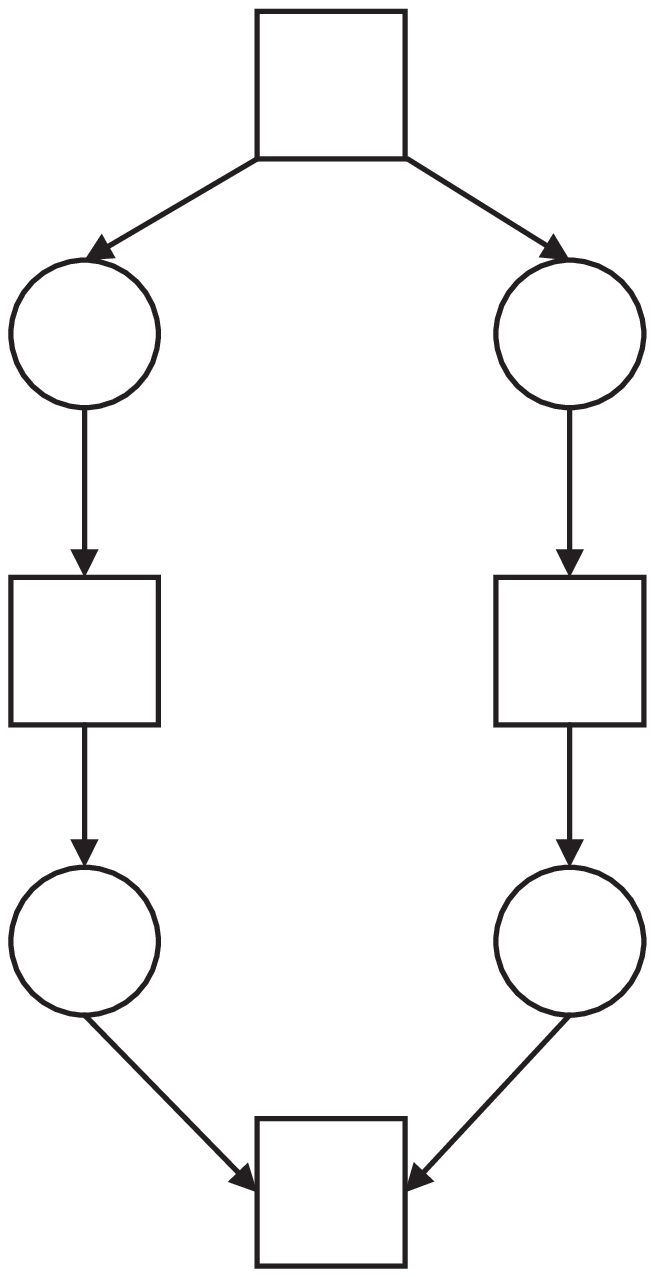}}

	\end{minipage}
	\begin{minipage}[c]{.24\linewidth}
		\centering
		\subfigure[]{
			\label{fig:subfig:16c} 
			\includegraphics[width=0.5\textwidth]{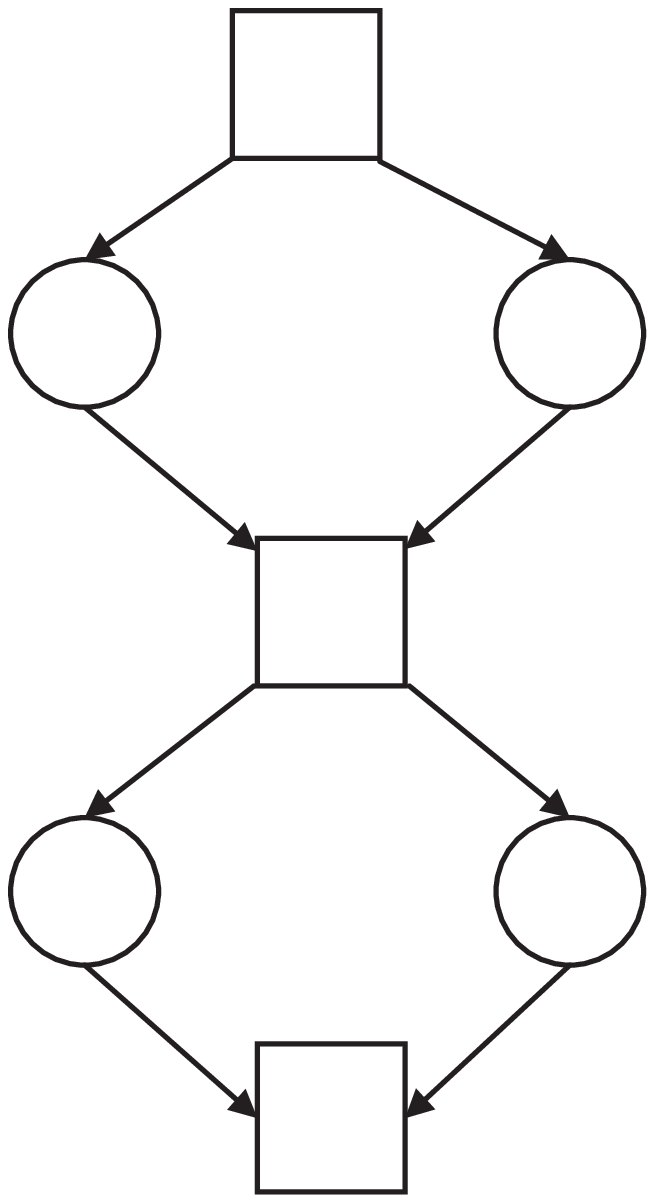}}

	\end{minipage}
	\begin{minipage}[c]{.24\linewidth}
		\centering
		\subfigure[]{
			\label{fig:subfig:16d} 
			\includegraphics[width=0.5\textwidth]{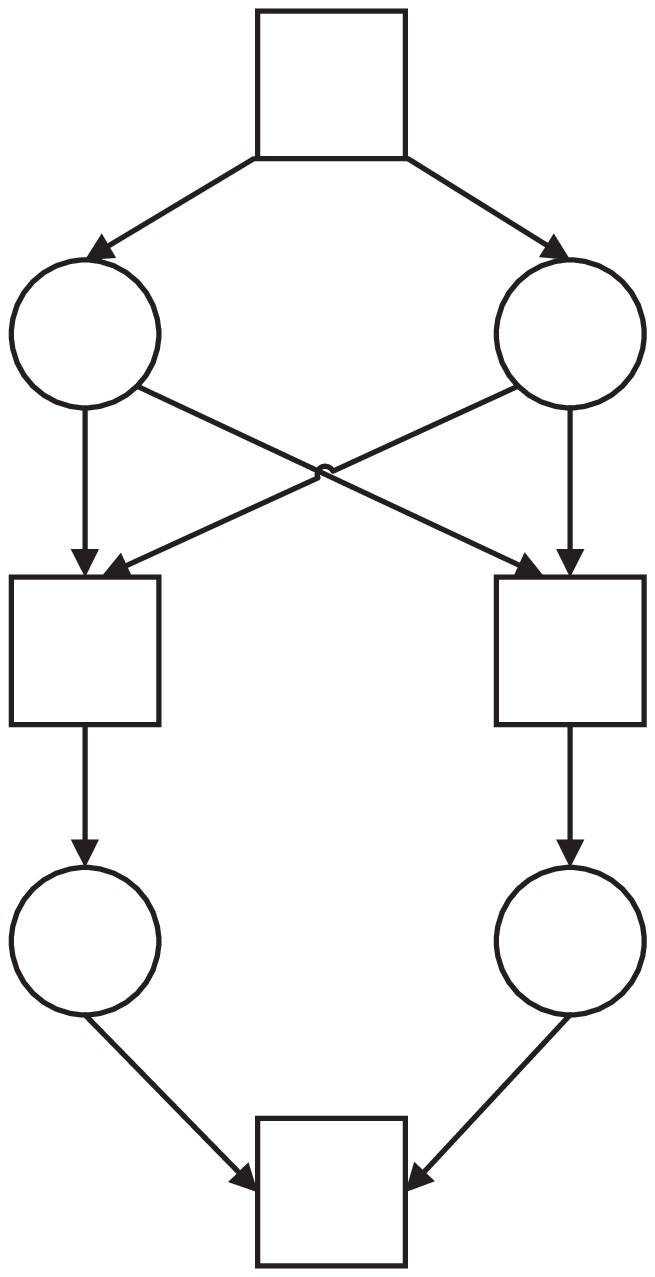}}

	\end{minipage}
	\caption{S\_graph, T\_graph, free-choice net, occurrence net }
	\label{16}
\end{figure}
\begin{itemize}
	\item Petri Nets with Inhibitor Arcs 
\end{itemize}
The addition of inhibitor arcs is a significant extension of Petri net \cite{peterson_petri_1981}. Agerwala has shown that Petri nets extended in this manner are equivalence to the Turing machine \cite{agerwala_complete_1974}. Petri nets with inhibitor arcs can be represented as ($S,T;F,I$), where ($S,T;F$) is a net, and $I$ is a set of inhibitor arcs.  $\forall x \in F, x \not \in I$. The effect of an inhibitor arc is opposite to that of the arcs in $F$. While an arc of $F$ leading from an established condition to an event means allowing it to occur, an inhibitor arc leading from an established condition to an event means probiting this event to occur. 

Since there exist two different kinds of arcs, the edge tagging in net graphs should be modified (i.e., separate the inhibitor arcs from the ordinary arcs). In each triple of edge tagging, we use double sign “$--$” mean the input condition is an inhibitor's input/output conditions. Figure ~\ref{inhibitor} shows an example of Petri net with inhibitor arcs and the corresponding net graph. While an arc of F is represented with an arrow head $'\to'$, an inhibitor arc of $I$ is represented with arrow circle '-o'.

\begin{figure}[H]
	
	\begin{minipage}[c]{0.45\linewidth}
		\centering
		\subfigure[A Petri net with inhibitor arcs]{
			\label{inhibitornet}
			\includegraphics[width=\textwidth]{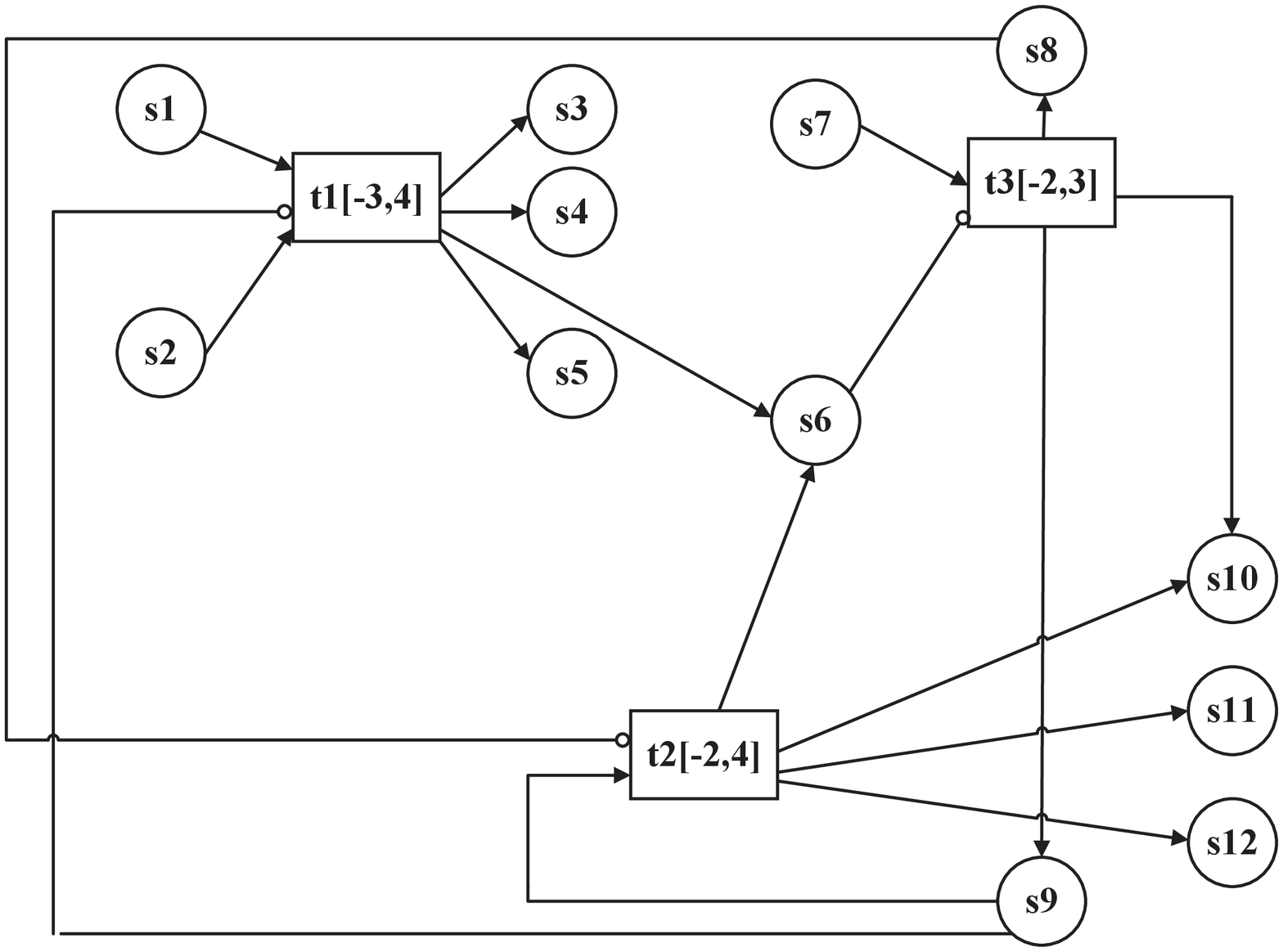}}
	\end{minipage}
	\begin{minipage}[c]{0.45\linewidth}
		\centering
		\subfigure[The correponding net graph]{
			\label{inhibitorng} 
			\includegraphics[width=\textwidth]{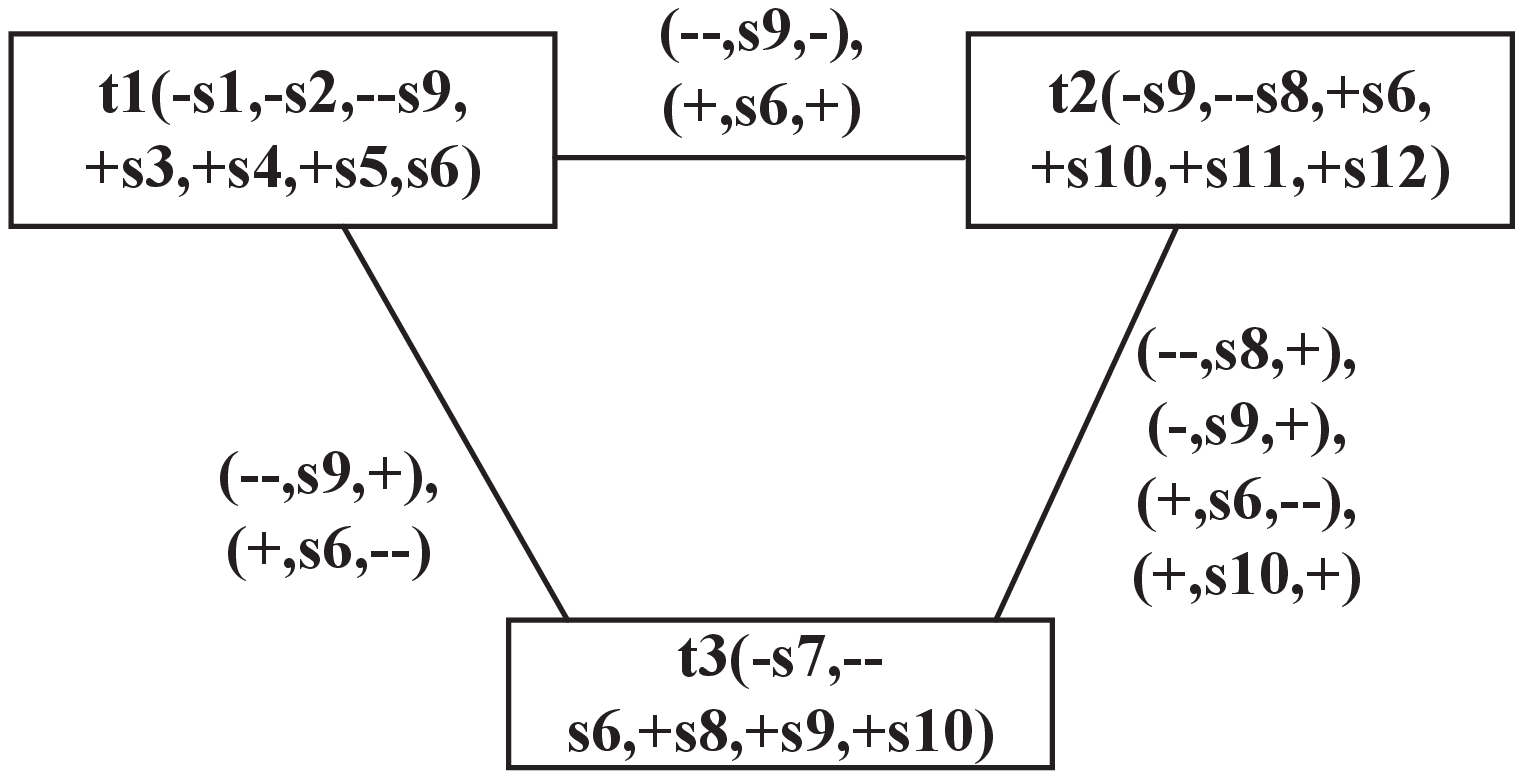}}

	\end{minipage}
	
	\caption{An example of Petri net with inhibitor arcs and the net graph}
	\label{inhibitor}
\end{figure}
\begin{itemize}
	\item Petri Nets with tokens
\end{itemize}

Petri nets with tokens are usually called Petri net systems. However we don’t use this terminology here, because true Petri net systems allow the tokens flowing through the whole net. In this sense the locality of a subnet is lost. Everything becomes global. The problem of frequent subnet mining doesn’t exist anymore. This is why we did not discuss frequent Petri net systems mining in this paper. 

\section{Conculsion} \label{sec:Conculsion}

The key contributions of this paper are summarized as follows.
\begin{itemize}
	\item We introduced the concept of complete subnets and take it as the basis of sub-Petri net mining. In this way we distinguished frequent Petri net mining from frequent subgraph mining in a theoretically strict way. 
	\item We introduced a new data structure called net graph to represent a pure C/E net in a pseudo-graph form in order to reduce the complexity of net structure analysis drastically, even exponentially. We proved that there is a bijective mapping between a pure C/E net and a net graph which is structure preserving. 
	\item We presented an algorithm PSpan1 for mining frequent complete subnets from pure C/E nets based on net graph representation. To our best knowledge, this is the first algorithm that can discover frequent subnets from a large set of Petri nets.
	\item We presented another algorithm PSpan2 based on a pseudo-graph representation which makes a direct use of gSpan strategy and serves as a baseline algorithm for PSpan1.
	\item A complexity analysis showed that PSpan1, PSpan2 and gSpan have roughly the same complexity, while an ideal experiment showed that DSpan would be exponentially more complex than PSpan where DSpan is a frequent subnet miner working directly on C/E net representation.
	\item We implemented an effective method for randomly generating large-scale test sets of connected pure C/E nets at thousands scale. We used net planting method to check the correctness and completeness of data mining. Results of experiments also confirmed our complexity analysis.
	\item We extended the above results to eight other subclasses of Petri nets and showed that our PSpan algorithm can be applied to all these eight subclasses of Petri nets with only a minor modification. 
\end{itemize}

\bibliography{review}

\begin{thebibliography}{10}

\bibitem{aalst_workflow_2004}
W.~van~der Aalst, T.~Weijters, and L.~Maruster.
\newblock Workflow mining: discovering process models from event logs.
\newblock {\it IEEE Transactions on Knowledge and Data Engineering},
  16(9):1128--1142, September 2004.

\bibitem{agerwala_complete_1974}
T.~Agerwala.
\newblock Complete model for representing the coordination of asynchronous
  processes.
\newblock 1974.
\newblock Publisher: Johns Hopkins University.

\bibitem{al-naymat_enumeration_2008}
G.~Al-Naymat.
\newblock Enumeration of maximal clique for mining spatial co-location
  patterns.
\newblock In {\it 2008 {IEEE}/{ACS} {International} {Conference} on {Computer}
  {Systems} and {Applications}}, pages 126--133, March 2008.

\bibitem{belhajjame_keyword-based_2017}
Khalid Belhajjame, Daniela Grigori, Mariem Harmassi, and Manel Ben~Yahia.
\newblock Keyword-{Based} {Search} of {Workflow} {Fragments} and {Their}
  {Composition}.
\newblock In Ngoc~Thanh Nguyen, Ryszard Kowalczyk, Alexandre~Miguel Pinto, and
  Jorge Cardoso, editors, {\it Transactions on {Computational} {Collective}
  {Intelligence} {XXVI}}, Lecture {Notes} in {Computer} {Science}, pages
  67--90. Springer International Publishing, Cham, 2017.

\bibitem{borgelt_mining_2002}
C.~Borgelt and M.~R. Berthold.
\newblock Mining molecular fragments: finding relevant substructures of
  molecules.
\newblock In {\it 2002 {IEEE} {International} {Conference} on {Data} {Mining},
  2002. {Proceedings}.}, pages 51--58, December 2002.

\bibitem{chapela-campa_towards_2017}
Chapela-Campa, Manuel Mucientes, and Manuel Lama.
\newblock Towards the {Extraction} of {Frequent} {Patterns} in {Complex}
  {Process} {Models}.
\newblock In {\it Jornadas de {Ciencia} e {Ingeniería} de {Servicios}}, 2017.

\bibitem{chapela-campa_discovering_2017}
David Chapela-Campa, Manuel Mucientes, and Manuel Lama.
\newblock Discovering {Infrequent} {Behavioral} {Patterns} in {Process}
  {Models}.
\newblock In Josep Carmona, Gregor Engels, and Akhil Kumar, editors, {\it
  Business {Process} {Management}}, Lecture {Notes} in {Computer} {Science},
  pages 324--340, Cham, 2017. Springer International Publishing.

\bibitem{cheong_psm-flow_2018}
Chin~Wang Cheong, Daniel Garijo, Cheung~Kwok Wai, and Yolanda Gil.
\newblock {PSM}-{Flow}: {Probabilistic} {Subgraph} {Mining} for {Discovering}
  {Reusable} {Fragments} in {Workflows}.
\newblock In {\it 2018 {IEEE}/{WIC}/{ACM} {International} {Conference} on {Web}
  {Intelligence} ({WI})}, pages 166--173. IEEE, 2018.

\bibitem{_three-layer_2007}
Shikun~Zhang Congyi~Yuan, Wen~Zhao and Yu~Huang.
\newblock A {Three}-{Layer} {Model} for {Business} {Processes}-{Process}
  {Logic},{Case} {Semantics} and {Workflow} {Management}.
\newblock {\it Journal of Computer Science \& Technology}, (03):410--425, 2007.

\bibitem{cook_substructure_1994}
Diane~J. Cook and Lawrence~B. Holder.
\newblock Substructure {Discovery} {Using} {Minimum} {Description} {Length} and
  {Background} {Knowledge}.
\newblock {\it Journal of Artificial Intelligence Research}, 1:231--255, 1994.

\bibitem{dalmas_heuristic_2018}
Benjamin Dalmas, Niek Tax, and Sylvie Norre.
\newblock Heuristic {Mining} {Approaches} for {High}-{Utility} {Local}
  {Process} {Models}.
\newblock In Maciej Koutny, Lars~Michael Kristensen, and Wojciech Penczek,
  editors, {\it Transactions on {Petri} {Nets} and {Other} {Models} of
  {Concurrency} {XIII}}, Lecture {Notes} in {Computer} {Science}, pages 27--51.
  Springer, Berlin, Heidelberg, 2018.

\bibitem{desel_what_2001}
Jörg Desel and Gabriel Juhás.
\newblock “{What} {Is} a {Petri} {Net}?” {Informal} {Answers} for the
  {Informed} {Reader}.
\newblock In Hartmut Ehrig, Julia Padberg, Gabriel Juhás, and Grzegorz
  Rozenberg, editors, {\it Unifying {Petri} {Nets}: {Advances} in {Petri}
  {Nets}}, Lecture {Notes} in {Computer} {Science}, pages 1--25. Springer,
  Berlin, Heidelberg, 2001.

\bibitem{deshpande_frequent_2005}
M.~Deshpande, M.~Kuramochi, N.~Wale, and G.~Karypis.
\newblock Frequent substructure-based approaches for classifying chemical
  compounds.
\newblock {\it IEEE Transactions on Knowledge and Data Engineering},
  17(8):1036--1050, August 2005.
\newblock Conference Name: IEEE Transactions on Knowledge and Data Engineering.

\bibitem{diamantini_pattern_2013}
Claudia Diamantini, Laura Genga, Domenico Potena, and Emanuele Storti.
\newblock Pattern discovery from innovation processes.
\newblock In {\it 2013 {International} {Conference} on {Collaboration}
  {Technologies} and {Systems} ({CTS})}, pages 457--464, May 2013.
\newblock ISSN: null.

\bibitem{diamantini_discovering_2015}
Claudia Diamantini, Laura Genga, Domenico Potena, and Emanuele Storti.
\newblock Discovering {Behavioural} {Patterns} in {Knowledge}-{Intensive}
  {Collaborative} {Processes}.
\newblock In Annalisa Appice, Michelangelo Ceci, Corrado Loglisci, Giuseppe
  Manco, Elio Masciari, and Zbigniew~W. Ras, editors, {\it New {Frontiers} in
  {Mining} {Complex} {Patterns}}, Lecture {Notes} in {Computer} {Science},
  pages 149--163, Cham, 2015. Springer International Publishing.

\bibitem{elseidy_grami:_2014}
Mohammed Elseidy, Ehab Abdelhamid, Spiros Skiadopoulos, and Panos Kalnis.
\newblock {GraMi}: {Frequent} {Subgraph} and {Pattern} {Mining} in a {Single}
  {Large} {Graph}.
\newblock {\it Proc. VLDB Endow.}, 7(7):517--528, March 2014.

\bibitem{garey1974some}
M~R Garey, David~S Johnson, and Larry Stockmeyer.
\newblock Some simplified np-complete problems.
\newblock pages 47--63, 1974.

\bibitem{garijo_mining_2015}
Daniel Garijo.
\newblock {\it Mining abstractions in scientific workflows}.
\newblock PhD thesis, Ph. D. Dissertation. Departamento de Inteligencia
  Artficial Escuela Técnica, 2015.

\bibitem{garijo_fragflow_2014}
Daniel Garijo, Oscar Corcho, Yolanda Gil, Boris~A. Gutman, Ivo~D. Dinov, Paul
  Thompson, and Arthur~W. Toga.
\newblock {FragFlow} {Automated} {Fragment} {Detection} in {Scientific}
  {Workflows}.
\newblock In {\it 2014 {IEEE} 10th {International} {Conference} on
  e-{Science}}, pages 281--289, Sao Paulo, Brazil, October 2014. IEEE.

\bibitem{greco_mining_2005}
G.~Greco, A.~Guzzo, G.~Manco, and D.~Sacca.
\newblock Mining and reasoning on workflows.
\newblock {\it IEEE Transactions on Knowledge and Data Engineering},
  17(4):519--534, April 2005.

\bibitem{greco_mining_2006}
Gianluigi Greco, Antonella Guzzo, Giuseppe Manco, Luigi Pontieri, and Domenico
  Saccà.
\newblock Mining {Constrained} {Graphs}: {The} {Case} of {Workflow} {Systems}.
\newblock In Jean-François Boulicaut, Luc De~Raedt, and Heikki Mannila,
  editors, {\it Constraint-{Based} {Mining} and {Inductive} {Databases}},
  Lecture {Notes} in {Computer} {Science}, pages 155--171, Berlin, Heidelberg,
  2006. Springer.

\bibitem{hack_analysis_1972}
Michel Henri~Theodore Hack.
\newblock Analysis of {Production} {Schemata} by {Petri} {Nets}.
\newblock Technical Report MAC-TR-94, MASSACHUSETTS INST OF TECH CAMBRIDGE
  PROJECT MAC, February 1972.

\bibitem{huan_spin:_2004}
Jun Huan, Wei Wang, Jan Prins, and Jiong Yang.
\newblock Spin: {Mining} maximal frequent subgraphs from graph databases.
\newblock 2004.

\bibitem{inokuchi_apriori-based_2000}
Akihiro Inokuchi, Takashi Washio, and Hiroshi Motoda.
\newblock An {Apriori}-{Based} {Algorithm} for {Mining} {Frequent}
  {Substructures} from {Graph} {Data}.
\newblock In Djamel~A. Zighed, Jan Komorowski, and Jan Żytkow, editors, {\it
  Principles of {Data} {Mining} and {Knowledge} {Discovery}}, Lecture {Notes}
  in {Computer} {Science}, pages 13--23. Springer Berlin Heidelberg, 2000.

\bibitem{lin_properties_2001}
Changjun Jiang and Weiming Lu.
\newblock On {Properties} of {Concurrent} {System} {Based} on {Petri} {Net}
  {Language}.
\newblock {\it Journal of Software}, 12(04):512--518, 2001.

\bibitem{kuramochi_frequent_2001}
M.~Kuramochi and G.~Karypis.
\newblock Frequent subgraph discovery.
\newblock In {\it Proceedings 2001 {IEEE} {International} {Conference} on
  {Data} {Mining}}, pages 313--320, San Jose, CA, USA, 2001. IEEE Comput. Soc.

\bibitem{kuramochi_grew_2004}
Michihiro Kuramochi and George Karypis.
\newblock {GREW} - {A} scalable frequent subgraph discovery algorithm.
\newblock In {\it Proceedings - {Fourth} {IEEE} {International} {Conference} on
  {Data} {Mining}, {ICDM} 2004}, pages 439--442, December 2004.

\bibitem{lakshmi_efficient_2013}
K.~Lakshmi and T.~Meyyappan.
\newblock Efficient {Algorithm} for {Mining} {Frequent} {Subgraphs} ({Static}
  and {Dynamic}) based on {gSpan}.
\newblock {\it International Journal of Computer Applications}, 63:9--12,
  February 2013.

\bibitem{leemans_discovery_2015}
Maikel Leemans and Wil M.~P. van~der Aalst.
\newblock Discovery of {Frequent} {Episodes} in {Event} {Logs}.
\newblock In Paolo Ceravolo, Barbara Russo, and Rafael Accorsi, editors, {\it
  Data-{Driven} {Process} {Discovery} and {Analysis}}, Lecture {Notes} in
  {Business} {Information} {Processing}, pages 1--31. Springer International
  Publishing, 2015.

\bibitem{leung_exploring_2010}
C.~K. Leung and C.~L. Carmichael.
\newblock Exploring {Social} {Networks}: {A} {Frequent} {Pattern}
  {Visualization} {Approach}.
\newblock In {\it 2010 {IEEE} {Second} {International} {Conference} on {Social}
  {Computing}}, pages 419--424, August 2010.

\bibitem{venkatasubramanian_reafum:_2015}
Ruirui Li and Wei Wang.
\newblock {REAFUM}: {Representative} {Approximate} {Frequent} {Subgraph}
  {Mining}.
\newblock In Suresh Venkatasubramanian and Jieping Ye, editors, {\it
  Proceedings of the 2015 {SIAM} {International} {Conference} on {Data}
  {Mining}}, pages 757--765. Society for Industrial and Applied Mathematics,
  Philadelphia, PA, June 2015.

\bibitem{li_efficient_2007}
XT~Li, JZ~Li, and H~Gao.
\newblock An {Efficient} {Frequent} {Subgraph} {Mining} {Algorithm}.
\newblock {\it Journal of Software}, 18(10):2469--2480, 2007.

\bibitem{lu_exact_2007}
Ruqian Lu, Caiyan Jia, Shaofang Zhang, Lusheng Chen, and Hongyu Zhang.
\newblock An {Exact} {Data} {Mining} {Method} for {Finding} {Center} {Strings}
  and {All} {Their} {Instances}.
\newblock {\it IEEE Transactions on Knowledge and Data Engineering},
  19(4):509--522, April 2007.
\newblock Conference Name: IEEE Transactions on Knowledge and Data Engineering.

\bibitem{lu_petri_1994}
Weiming Lu and Chuang Lin.
\newblock Petri {Nets}: {Opportunities} and {Challenges}.
\newblock {\it COMPUTER SCIENCE}, 021(004):1--5, 1994.

\bibitem{murata_petri_1989}
T.~Murata.
\newblock Petri nets: {Properties}, analysis and applications.
\newblock {\it Proceedings of the IEEE}, 77(4):541--580, April 1989.

\bibitem{nijssen_quickstart_2004}
Siegfried Nijssen and Joost~N. Kok.
\newblock A quickstart in frequent structure mining can make a difference.
\newblock In {\it In {Proc}. of the 10th {ACM} {SIGKDD} {International}
  {Conference} on {Knowledge} {Discovery} and {Data} {Mining} ({KDD}-2004},
  pages 647--652, 2004.

\bibitem{pei_mining_2005}
J.~Pei, D.~Jiang, and A.~Zhang.
\newblock Mining cross-graph quasi-cliques in gene expression and protein
  interaction data.
\newblock In {\it 21st {International} {Conference} on {Data} {Engineering}
  ({ICDE}'05)}, pages 353--356, April 2005.

\bibitem{peterson_petri_1977}
James~L. Peterson.
\newblock Petri nets.
\newblock {\it ACM Computing Surveys (CSUR)}, 9(3):223--252, 1977.
\newblock ISBN: 0360-0300 Publisher: ACM New York, NY, USA.

\bibitem{peterson_petri_1981}
James~Lyle Peterson.
\newblock {\it Petri {Net} {Theory} and the {Modeling} of {Systems}}.
\newblock Prentice Hall PTR, Upper Saddle River, NJ, USA, 1981.

\bibitem{petri_concepts_1973}
C.~A. Petri.
\newblock Concepts of {Net} {Theory}.
\newblock In {\it mathematical foundations of computer science}, pages
  137--146, 1973.

\bibitem{petri_kommunikation_1962}
Carl~Adam Petri.
\newblock {\it Kommunikation mit {Automaten}}.
\newblock PhD thesis, 1962.

\bibitem{reisig_lectures_1998}
Wolfgang Reisig and Grzegorz Rozenberg.
\newblock {\it Lectures on {Petri} {Nets} {I}: {Basic} {Models}: {Advances} in
  {Petri} {Nets}}.
\newblock Springer Science \& Business Media, November 1998.
\newblock Google-Books-ID: 4BbFfLMZqnYC.

\bibitem{silva_half_2013}
Manuel Silva.
\newblock Half a century after {Carl} {Adam} {Petri}’s {Ph}.{D}. thesis: {A}
  perspective on the field.
\newblock {\it Annual Reviews in Control}, 37(2):191--219, December 2013.

\bibitem{tapiaflores_discovering_2018}
Tonatiuh Tapiaflores, Ernesto Lopezmellado, Ana~Paula Estradavargas, and
  Jeanjacques Lesage.
\newblock Discovering {Petri} {Net} {Models} of {Discrete}-{Event} {Processes}
  by {Computing} {T}-{Invariants}.
\newblock {\it IEEE Transactions on Automation Science and Engineering},
  15(3):992--1003, 2018.

\bibitem{tax_mining_2016}
Niek Tax, Natalia Sidorova, Reinder Haakma, and Wil~M.P. van~der Aalst.
\newblock Mining local process models.
\newblock {\it Journal of Innovation in Digital Ecosystems}, 3(2):183--196,
  December 2016.

\bibitem{tax_localprocessmodeldiscovery:_2018}
Niek Tax, Natalia Sidorova, Wil M.~P. van~der Aalst, and Reinder Haakma.
\newblock {LocalProcessModelDiscovery}: {Bringing} {Petri} {Nets} to the
  {Pattern} {Mining} {World}.
\newblock In Victor Khomenko and Olivier~H. Roux, editors, {\it Application and
  {Theory} of {Petri} {Nets} and {Concurrency}}, Lecture {Notes} in {Computer}
  {Science}, pages 374--384, Cham, 2018. Springer International Publishing.

\bibitem{tax_n_use_2017}
{Tax, N.}, {Genga, L.}, {Zannone, N.}, {Ceravolo, P.}, {van Keulen, M.},
  {Stoffel, K.}, {Information Systems WSK\&I}, and {Security W\&I}.
\newblock On the use of hierarchical subtrace mining for efficient local
  process model mining.
\newblock In {\it {CEUR}-ws.org}, volume 2016. CEUR-WS.org, December 2017.

\bibitem{uno2016mining}
Takeaki Uno and Yushi Uno.
\newblock Mining preserving structures in a graph sequence.
\newblock {\it Theoretical Computer Science}, 654:155--163, 2016.

\bibitem{valluri_margin:_2010}
Ramachandra~Satyanarayana Valluri, Lini~T. Thomas, and Kamalakar Karlapalem.
\newblock {MARGIN}: {Maximal} frequent subgraph mining.
\newblock {\it ACM Transactions on Knowledge Discovery from Data}, 4(3), 2010.

\bibitem{van_der_aalst_prom:_2009}
Wil~MP Van~der Aalst, Boudewijn~F. van Dongen, Christian~W. Günther, Anne
  Rozinat, Eric Verbeek, and Ton Weijters.
\newblock {ProM}: {The} process mining toolkit.
\newblock {\it BPM (Demos)}, 489(31):2, 2009.

\bibitem{vander_aalst_decomposing_2013}
Wil M.~P. van der Aalst.
\newblock Decomposing {Petri} nets for process mining: {A} generic approach.
\newblock {\it Distributed and Parallel Databases}, 31(4):471--507, December
  2013.

\bibitem{wen_mining_2007}
Lijie Wen, Wil M.~P. van~der Aalst, Jianmin Wang, and Jiaguang Sun.
\newblock Mining process models with non-free-choice constructs.
\newblock {\it Data Mining and Knowledge Discovery}, 15(2):145--180, October
  2007.

\bibitem{wen_novel_2009}
Lijie Wen, Jianmin Wang, Wil M.~P. van~der Aalst, Biqing Huang, and Jiaguang
  Sun.
\newblock A novel approach for process mining based on event types.
\newblock {\it Journal of Intelligent Information Systems}, 32(2):163--190,
  April 2009.

\bibitem{wu_introduction_2006}
Zhehui Wu.
\newblock {\it Introduction to {Petri} {Nets}}.
\newblock Beijing: CHINA MACHINE PRESS, April 2006.

\bibitem{xifeng_yan_gspan:_2002}
{Xifeng Yan} and {Jiawei Han}.
\newblock {gSpan}: graph-based substructure pattern mining.
\newblock In {\it 2002 {IEEE} {International} {Conference} on {Data} {Mining},
  2002. {Proceedings}.}, pages 721--724, Maebashi City, Japan, 2002. IEEE
  Comput. Soc.

\bibitem{yang_computational_2006}
Guizhen Yang.
\newblock Computational aspects of mining maximal frequent patterns.
\newblock {\it Theoretical Computer Science}, 362(1):63--85, October 2006.

\bibitem{yang_combining_2009}
Tianbao Yang, Rong Jin, Yun Chi, and Shenghuo Zhu.
\newblock Combining link and content for community detection: a discriminative
  approach.
\newblock In {\it Proceedings of the 15th {ACM} {SIGKDD} international
  conference on {Knowledge} discovery and data mining - {KDD} '09}, page 927,
  Paris, France, 2009. ACM Press.

\bibitem{yuan_application_2013}
Chong-Yi Yuan.
\newblock {\it Application of {Petri} {Nets}}.
\newblock SCIENCE PRESS, Beijing, China, 2013.

\end{thebibliography}

\end{document}